\documentclass[lettersize,journal]{IEEEtran}
\usepackage{amsmath,amsfonts}
\usepackage{algorithmic}
\usepackage{algorithm}
\usepackage{array}
\usepackage{textcomp}
\usepackage{stfloats}
\usepackage{url}
\usepackage{verbatim}
\usepackage{graphicx}
\usepackage{cite}
\usepackage{subfigure}
\usepackage{bm}

\usepackage{amssymb}
\usepackage{hyperref}
\usepackage{amsthm}
\usepackage{todonotes}
\usepackage{mathtools}
 \usepackage{dutchcal}
 \usepackage{stmaryrd}
 \newtheorem{definition}{Definition}
 
 \newtheorem{lemma}{Lemma}
 \def\R{\mathbb{R}}
 \def\Z{\mathbb{Z}}

\begin{document}

\title{Estimating the Coverage Measure and the Area Explored by a Line-Sweep Sensor on the Plane}

\author{Maria Costa Vianna, Eric Goubault, Luc Jaulin, Sylvie Putot~\IEEEmembership{}
\thanks{Maria Costa Vianna, Eric Goubault and Sylvie Putot are with the Computer Science Laboratory of the École Polytechnique (LIX), \'Ecole Polytechnique, Palaiseau, France (e-mail: 
\{costavianna,goubault,putot\}@lix.polytechnique.fr).

Luc Jaulin is with the Lab-STICC, ENSTA Bretagne, Brest, France
(e-mail: luc.jaulin@ensta-bretagne.fr). }
\thanks{
}}


\IEEEpubid{}

\maketitle

\begin{abstract}

This paper presents a method for determining the area explored by a line-sweep sensor during an area-covering mission in a two-dimensional plane. Accurate knowledge of the explored area is crucial for various applications in robotics, such as mapping, surveillance, and coverage optimization. The proposed method leverages the concept of coverage measure of the environment and its relation to the topological degree in the plane, to estimate the extent of the explored region. In addition, we extend the approach to uncertain coverage measure values using interval analysis. This last contribution allows for a guaranteed characterization of the explored area, essential considering the often critical character of area-covering missions. Finally, this paper also proposes a novel algorithm for computing the topological degree in the 2-dimensional plane, for all the points inside an area of interest, which differs from existing solutions that compute the topological degree for single points. The applicability of the method is evaluated through a real-world experiment.

\end{abstract}


\section{Introduction}
Mobile robots are increasingly being used to carry out dangerous tasks that otherwise would put human lives at risk, such as bomb disposal, firefighting, and search and rescue missions. Their use in these situations can considerably reduce the risk to human workers while providing more detailed and accurate information about the situation. Additionally, mobile robots can be equipped with specialized tools, such as cameras, grippers, and cutting devices, that enable them to perform a wide range of tasks that would be difficult or impossible for humans to do. In the context of these operations, the robotic platform often needs to perform an area-covering mission. During these missions, a designated part of the robot's environment is thoroughly searched or monitored to develop a complete understanding of the situation or identify potential threats or opportunities.

Determining the area explored by a mobile robot during an area-covering mission is important to establish if the mission is successful. It is also essential for validating path-planning algorithms that will lead to complete coverage  of an area of interest \cite{surveycoverage} or complete avoidance of an area of risk. Overall, determining the explored area is essential for ensuring efficient and safe operations, planning future actions, and gaining valuable insights from the acquired data.

In addition, we are also interested in determining the coverage measure of a point in the environment. The coverage measure represents how many times this point was covered by the robot's sensors or tools, in other words, how many times it was explored. 

Counting the number of times an area was explored is of interest for different reasons, for example, when assessing revisiting missions. In these missions the robot is required to come back to a previous point, therefore to revisit it, to improve the quality of information collected around this point through redundancy. Indeed, studies have shown that target classification improves dramatically when a multi-view approach is adopted. Usually, single-view approaches do not provide enough information to make a confident identification with, for example, Synthetic Aperture Sonars (SAS) \cite{williams} and Synthetic Aperture Radars \cite{ding}. A multi-view method is also essential when recognizing or reconstructing 3-dimensional objects from 2-dimensional data such as camera images \cite{3dcamera}. In these examples, counting how many times a point or an area, as a set of points, has already been explored will be essential to determine the mission completeness. On the contrary, if the robot is not supposed to cover areas previously visited, the coverage measure will be useful for planning optimal paths, reducing unnecessary effort.

In this context, in this work, we present a technique for quantifying the extent of coverage achieved by a mobile robot during a sweep exploration in a two-dimensional environment. Sweep exploration refers to missions where the robot uses a line-sweep sensor. Line-sweep sensors are one-dimensional sensors that provide data along a single axis and must sweep the environment in order to create a two-dimensional representation of its surroundings. With this purpose, we establish a relation between the exploration problem and the topological degree and we demonstrate how it can be used to determine the coverage measure. 

Topological concepts have already been explored for counting \cite{ghristtargetenumeration} and for addressing coverage problems in robotics contexts, e.g.  \cite{ghristcoveragehomology}, \cite{ghristcoveragepersistence}. The main advantage of the approach presented in this paper, is that we determine the number of times an area was explored, with the coverage measure, and different from more common approaches, such as grid-based analysis, our topological method does not require a previous discretization of the environment into fixed cells. We demonstrate that the whole environment can be characterized from very basic information on the robot's state and on the range of visibility of the exploration sensors, resulting in a method of low computational complexity. This approach has already been explored at \cite{oceans}, but here we deepen its mathematical definition and extend it to address previous limitations such as the coverage measure of points on the maximal range of visibility and of points that are swept on the opposite direction of movement.

We also  address the crucial issue of uncertainty in a robot's trajectory to achieve a guaranteed estimation of the explored area. In \cite{desrochers}, a method to estimate the explored area considering the uncertain state of a robot was presented. We extend their method by introducing the concept of uncertain coverage measure.

Our last contribution is an algorithm for computing the winding number of a continuous cycle with respect to all the point in the two-dimensional plane. Algorithms for general topological degree computation have already been proposed by different works \cite{stenger},\cite{franek}. However, methods available in the literature will compute the winding number of a cycle with respect to a single point, needing to be applied to each point individually for a full characterization of the plane. In this context, we present a set-membership approach that efficiently determines the winding number for a whole area of interest. The resulting algorithm and all the concepts defined in this work are applied to determine the area explored by a real autonomous underwater vehicle doing an exploration mission with two line-sweep sensors. 


\section{Problem Statement}
\label{pstatement}
We are interested in the problem of a mobile robot that explores an unknown planar environment. We assume that the robot's pose can be fully described by a function of time: $\bm{x}: \mathbb{R} \rightarrow \mathbb{R}^3$ that is at least $C^2$. The robot's visible area at time $t$ is a subset $\mathbb{V}(t) \subset \mathbb{R}^2$ of the environment that is sensed by the robot's embedded exteroceptive sensors. 

 \begin{figure}[h]
  \centering
  \subfigure[]{\includegraphics[scale=1.]{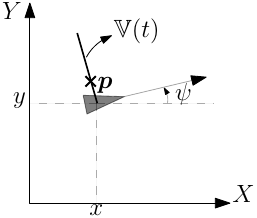}}\quad
  \subfigure[]{\includegraphics[scale=1.]{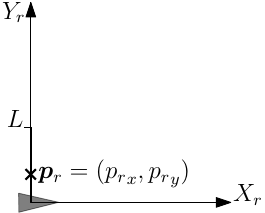}}
  \caption{(a): Mobile robot with a line sweep exploration sensor on the plane. At instant $t$ the point $\bm{p}$ is sensed by the robot ; (b): The point $\bm{p}_r$ is the representation of point $\bm{p}$ in the robot's coordinate frame $X_rY_r$. }
	\label{fig:line_sweep_v}
\end{figure}

 
 We define $\mathbb{V}$ as a set-valued function that depends on the robot's pose and the geometry and technology of the sensors employed. In this work, we focus on the problem of line-sweep exploration sensors and we treat the example of one that osculates the environment on the robot's left side as it moves around the plane, Figure \ref{fig:line_sweep_v}. In this context, the robot's pose at instant $t$ can be represented by the vector $$\bm{x}(t) = \begin{pmatrix} x(t) & y(t) & \psi(t) \end{pmatrix}^T$$ where the pair $(x,y)$ represents the robot's position in the plane and $\psi$ its orientation. Let  $L \in \mathbb{R}^+$ be the sensor's visible range, the visible set in this configuration can be defined as \begin{equation}
     \label{eq:line_sweep_v}
     \mathbb{V}(t) = \{ \bm{p} \in \mathbb{R}^2 | {p_r}_x = 0 \text{ and }  0 \leq {p_r}_y \leq L \}
 \end{equation} where \begin{equation}
    \label{eq:p_rob_frame}
    \bm{p}_r = \begin{pmatrix}
        {p_r}_x & {p_r}_y
    \end{pmatrix}^T = R^{-1}(\psi(t))(\bm{p} - \begin{pmatrix}
        x & y
    \end{pmatrix}^T)
\end{equation} represents in the robot's coordinate frame a point $\bm{p}$ in the environment and $R(\psi(t))$ is the rotation matrix associated with the robot's orientation angle $\psi(t)$.

\begin{figure}[h]
	\centering
	\includegraphics[scale = 0.6]{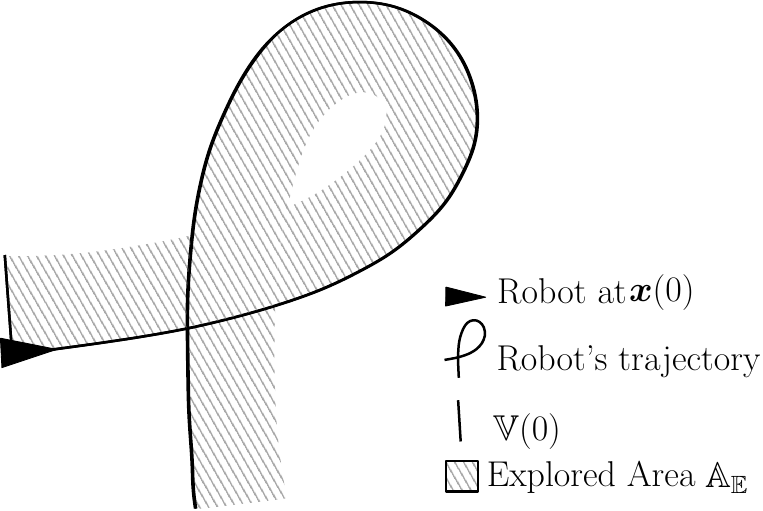}
	\caption{Area explored by a line-sweep sensor on the robot's left side along its trajectory.}
	\label{fig:line_sweep_ae}
\end{figure}

The set $\mathbb{A}_\mathbb{E}$ corresponds to the area explored by the robot during a time interval $[0,T]$, for some maximal value $T > 0$. It can be defined as the union of the robot's visible area along its trajectory \begin{equation}
	\label{eq:explored_area}
	\mathbb{A}_{\mathbb{E}} = \bigcup\limits_{t \in [0,T]} \mathbb{V}(t)
\end{equation}
Figure \ref{fig:line_sweep_ae} shows the resultant $\mathbb{A}_\mathbb{E}$ if we consider the illustrated robot's trajectory and the visible set function described by \eqref{eq:line_sweep_v}.

The robot's visibility region in this case can be parameterized by $u \in U \subseteq \mathbb{R}$. In the considered example $U =  [0,L]$ represents the lateral distance of a point in the visible area to the robot. We can define the sweep function $\bm{f}: U \times [0,T] \rightarrow \mathbb{R}^2$ as a continuously differentiable function whose image over the space $U \times t$, with $t \in [0,T]$, represents the visible area $\mathbb{V}(t)$,
\begin{equation}
	\mathbb{V}(t) = \bm{f}(U,t)
\end{equation}
\begin{figure}[h]
	\centering
	\includegraphics[scale = 0.45]{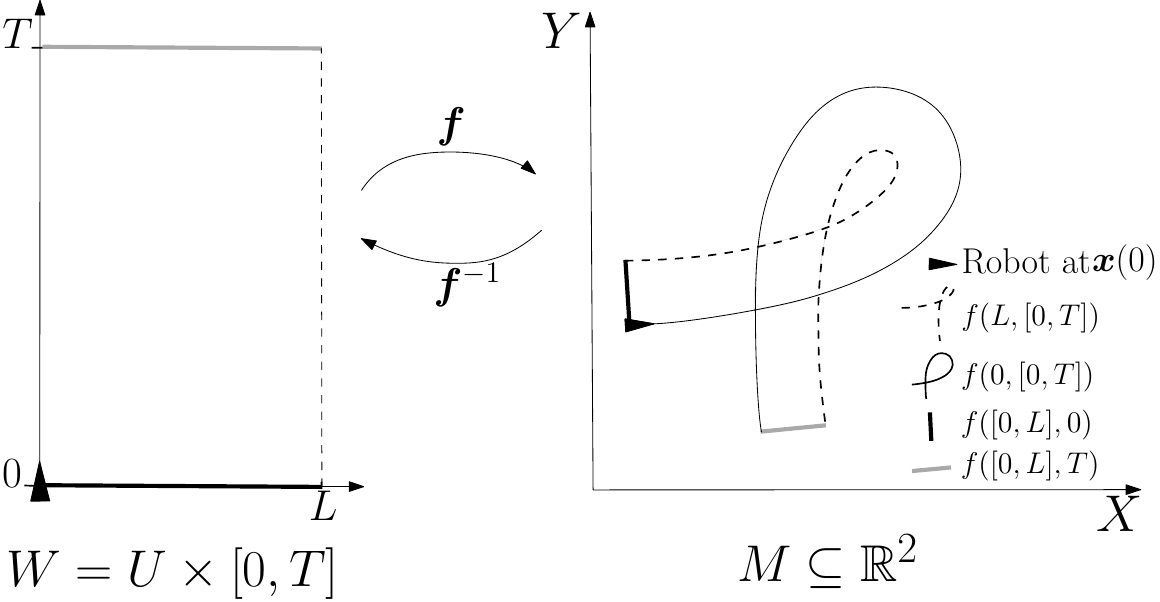}
	\caption{Waterfall and Mosaic Spaces for the line-sweep sensor example.}
	\label{fig:wf_mosaic_onesonar_left}
\end{figure}
By analogy to a common terminology adopted in sonar imagery \cite{waterfall}, we name space $W = U \times [0, T]$ the Waterfall Space. Points in $W$ are of the form $(u,t)$, $u$ representing the parameterization of the visible area, $t$ the time of exploration. All points $(u,t) \in W$ are points that were in the robot's visible area at least once and therefore, points that were explored during the mission. The robot's pose $\bm{x}$, its visible area $\mathbb{V}$ and $\mathbb{A}_{\mathbb{E}}$ are all defined inside an absolute coordinate system, the Mosaic Space $M \subseteq \mathbb{R}^2$ or the World Frame, as it is usually called in robotics. The sweep function $\bm{f}$ maps points from the Waterfall to the Mosaic space, Figure \ref{fig:wf_mosaic_onesonar_left}.

\begin{figure}[h]
	\centering
	\includegraphics[scale = 0.9]{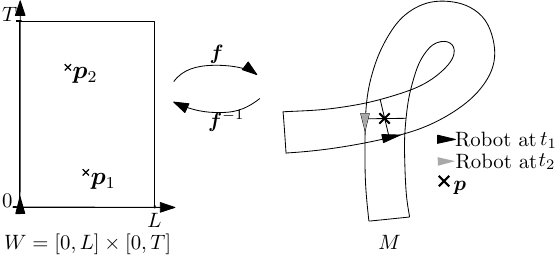}
	\caption{ Point $\bm{p}$ is revisited once during the mission and $\bm{f}^{-1}(\bm{p}) = \{ \bm{p}_1, \bm{p}_2\}$.}
	\label{fig:cm_onesonar_left}
\end{figure}

    The coverage measure, or how many times a point in the environment was explored by the robot during a mission, is given by the function $c_m: M \rightarrow \mathbb{N}_0$. A point is considered to be revisited if once in the robot's visibility range, it goes out of reach and then is sensed again later in time. In Figure \ref{fig:cm_onesonar_left}, for example, point $\bm{p}$ is sensed for the first time at instant $t_1$ and revisited at instant $t_2$, in this case, $c_m(\bm{p}) = 2$. 
    
    Let $det$ be the determinant function and $J_{\bm{f}}$ represents the Jacobian matrix of the sweep function. We adopt the following condition: \begin{equation}
    \label{eq:det}
        \forall \bm{w} \in W , det(J_{\bm{f}}(\bm{w})) > 0
    \end{equation}
    that implies that the robot is constantly moving and that the sensor sweeps the environment on the same direction of its advancement movement. By assuming this condition is met, we can say that the number of times that a point appears in the waterfall space corresponds to the number of times that this point was explored during a mission. If $Ker \ \bm{f}$ is the kernel of function $\bm{f}$, considering the definitions stated in this Section: for $\bm{p} \in M$, it can be concluded that \begin{equation}
    	\label{eq:def_cm}
    	c_m(\bm{p})  = \# Ker \ (\bm{f} - \bm{p})
    \end{equation}


The explored area $\mathbb{A}_{\mathbb{E}}$ can be characterized as the set of points that were sensed by the robot at least once and therefore in terms of the coverage measure of its points: 
\begin{equation}
	\label{eq:def_ae_cm}
	\mathbb{A}_{\mathbb{E}} = \{ \bm{p} \in M | c_m(\bm{p}) \geq 1\} 
\end{equation}

Describing the mosaic space using the coverage measure of its points is the method adopted in this work for defining the explored area. To achieve this, the following section establishes a connection between the topological degree and the coverage measure and this relation is explored with this purpose.

\section{Coverage Measure and Topological Degree}
In \cite{oceans} a relation between the coverage measure of a point in the plane and the topological degree has been explored. 
Here we give a general axiomatic definition of the notion of topological degree and recap the main properties that we use. 

\begin{definition}[Topological degree]
	\label{def:deg}
Let $D$ be an open subset of $\R^n$ and $\bm{f}$ a continuous function from its closure $\overline{D}$ to $\R^n$. A degree of $\bm{f}$ is a family of functions $deg: \ (\bm{f},D,\bm{p}) \rightarrow \Z$ for all $D$ open subsets of $\R^n$, $\bm{f}$ continuous and $\bm{p} \in \R^n\backslash \bm{f}(\partial D)$ such that: 
\begin{itemize}
    \item (identity) $deg(Id_{D},D,\bm{p})=1$ if $\bm{p} \in D$
    \item (excision) $deg(\bm{f},D,\bm{p})=deg(\bm{f},D_1,\bm{p})+deg(\bm{f},D_2,\bm{p})$ where $D_1$, $D_2$ are opens in $D$ with $\bm{p} \not \in \bm{f}(\overline{D}\backslash (D_1 \cup D_2))$
     \item (homotopy invariance) $deg(\bm{h}(\alpha,.),D,\bm{p}(\alpha))$ is independent of $\alpha$ for any homotopy $\bm{h}: \ [0,1] \times \overline{D} \rightarrow \R^n$, and $\bm{p}(\alpha)\not \in \bm{h}(\alpha,\partial D)$ for all $\alpha \in [0,1]$.
\end{itemize}
\end{definition}

When such a family of function exists, it is known to be unique \cite{mappingdegree}. In particular, when $\bm{f}$ is at least continuously differentiable, and $\bm{p}$ is a regular value of $\bm{f}$ (i.e. the determinant of the Jacobian of $\bm{f}$, $det(J_{\bm{f}})$, is non zero on each $\bm{d}$ with $\bm{f}(\bm{d})=\bm{p}$): 
\begin{equation}  \label{eq:topdeg_def}  deg(\bm{f},D,\bm{p})=\sum\limits_{\bm{d} \in \bm{f}^{-1}(\bm{p})} sign(det(J_{\bm{f}}(\bm{d})))\end{equation}

As well known in complex analysis, the topological degree of differentiable functions from the unit ball $D^2$ in $\R^2$ to $\R^2$ is linked to the winding number of $\bm{f}(\partial D^2)$. We are going to take the homological view on winding numbers in this paper. Let $S^1=\partial D^2$ be the 1-sphere, $\bm{p}$ a point in the interior of the image by $\bm{f}$ of $D^2$. Function $\bm{f}$ maps $S^1$, on a cycle in $\R^2$, and the winding number is the number of times this cycle turns around $\bm{p}$. By convention, counterclockwise turns count positively and clockwise turns negatively.  

\begin{definition}[Winding number]
\label{def:wn}
Let $\bm{f}: \ D^2 \rightarrow \mathbb{R}^2$ be a continuous function and $\bm{p} \in \bm{f}(D^2) \backslash \bm{f}(S^1)$. Consider its restriction $\bm{f}_{\mid S^1}: \ S^1 \rightarrow \mathbb{R}^2\backslash \{\bm{p}\}$. It induces a linear map in homology: 
$$\tilde{\bm{f}}: \ H_1(S^1) \rightarrow H_1(\mathbb{R}^2 \backslash \{\bm{p}\})$$
\noindent i.e. from $\mathbb{Z}$ to $\mathbb{Z}$, i.e. is of the form $\tilde{\bm{f}}(C)=\eta C$, where $C$ represents an equivalence class in $H_1(S^1)$. This $\eta$ is called the winding number of $\gamma=\bm{f}(S^1)$ around point $\bm{p} \in \bm{f}(D^2) \backslash \bm{f}(S^1)$. For all other points in $\mathbb{R}^2 \backslash \partial D^2$ the winding number is set to zero. 
\end{definition}

We can now state the relation between the topological degree and the winding number: 
\begin{lemma}
Let $\bm{f}$ be a continuously differentiable map from $D^2$ to $\mathbb{R}^2$ and let $\bm{y} \in \mathbb{R}^2 \backslash  \bm{f}(\partial D^2)$ such that $\bm{f}^{-1}(\bm{y})$ is finite and $\bm{y}$ is a regular point for $\bm{f}$. Then $deg(\bm{f},D^2,\bm{y})$ is equal to the winding number $\eta(\bm{f}(\partial D^2),\bm{y})$ of $\bm{f}(\partial D^2)$ at $\bm{y}$.
\end{lemma}

\begin{proof}
For all $\bm{y} \in {\mathbb{R}^2 \backslash \bm{f}(\partial {D^2})}$, either there exists no $\bm{d}$ such that $\bm{y}=f(\bm{d})$, or there exists a finite, non-zero number of $\bm{d}$, $\bm{d}_1,\ldots,\bm{d}_m$ in $D^2$, such that $\bm{f}(\bm{d}_i)=\bm{y}$.

In the first case, this means that both, $deg(\bm{f},D^2,\bm{y})$ is zero and $\bm{y}$ is in the complement of $\bm{f}(D^2)$ and the winding number $\eta(\bm{f}(\partial D^2),\bm{y})$ is also zero. 

In the second case, $\bm{y}$ being regular for $\bm{f}$, we have 
$deg(\bm{f},D,\bm{y})=\sum\limits_{i=1}^{m} sign(det(J_{\bm{f}}(\bm{d}_i)))$. Take small enough open neighborhoods $U_i$ of $\bm{d}_i$ in $D$ such that the sign of $det(J_{\bm{f}}(\bm{d}))$ is the same as the sign of $det(J_{\bm{f}}(\bm{d}_i))$ for all $\bm{d} \in U_i$. This is always possible since $J_{\bm{f}}$ is continuous. Note that this implies that $\bm{f}$ restricted to $U_i$ induces an homeomorphism onto its image. Also we can always choose the $U_i$ to have empty pairwise intersections and to have $\bm{f}$ being an homeomorphism from $\overline{U_i}$ onto its image, by taking them small enough (the $\bm{d}_i$ are isolated points within $D$). 

Now, the map $\tilde{\bm{f}}$ is the same as the map induced in homology $\tilde{\bm{f}}$ by $\bm{f}: \ D^2\backslash \bigcup\limits_{i=1}^m U_i \rightarrow \mathbb{R}^2 \backslash \{\bm{y}\}$. We note also that within $D^2\backslash\bigcup\limits_{i=1}^m U_i$, the cycle $\partial D^2$ is homologous to the sum of the $\partial(U_i)$, for $i=1,\ldots,m$. Hence $\tilde{\bm{f}}(\partial D^2)=\sum\limits_{i=1}^m \tilde{\bm{f}}(\partial(U_i))$. 

But $\bm{f}(\partial(U_i))$ is a Jordan curve homeomorphic (by $\bm{f}$) to $\partial(U_i)$, since we chose $U_i$ such that $\bm{f}$ restricted to $\overline{U_i}$ onto its image is a homeomorphism. Hence $\tilde{\bm{f}}(\partial U_i )$ is either plus or minus identity, according to the orientation of $\tilde{\bm{f}}(\partial U_i )$, i.e. $\tilde{\bm{f}}(\partial U_i)=sign(det(J_{\bm{f}}(\bm{d})))$ for any $\bm{d} \in U_i$, which we know is equal to $sign(det(J_{\bm{f}}(\bm{d}_i))$. Hence

$$\eta(\bm{f}(\partial D^2),\bm{y})=\sum\limits_{i=1}^m sign(det(J_{\bm{f}}(\bm{d}_i)))=deg(\bm{f},D^2,\bm{y})$$. 



\end{proof}

 \begin{figure}[h]
	\centering
	\includegraphics[scale = 0.5]{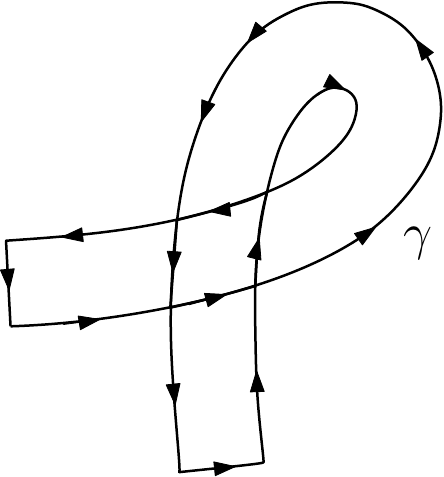}
	\caption{The sensor’s contour $\gamma$ for the mission represented in Figure \ref{fig:line_sweep_ae}.}
	\label{fig:contour}
\end{figure}

Now let $\bm{f}$ represent the sweep function, mapping from the Waterfall Space $W$, which is homeomorphic to $D^2$, to the Mosaic Space $M$. According to \eqref{eq:topdeg_def} and under hypothesis \eqref{eq:det}, for $\bm{p} \in \R^n\backslash \bm{f}(\partial W)$, \begin{equation}  deg(\bm{f},W,\bm{p}) =  \sum\limits_{\bm{w} \in \bm{f}^{-1}(\bm{p})} +1 = \# Ker \ (\bm{f} - \bm{p}) \end{equation} Finally, from \eqref{eq:def_cm}, it can be concluded that $deg(\bm{f},W,\bm{p}) = c_m(\bm{p})$. Moreover, from Definition \ref{def:wn}, \begin{equation}
	\label{eq:oceans_prop}
	\eta(\gamma,\bm{p}) =  c_m(\bm{p}),
\end{equation}
where $\gamma = \bm{f}(\partial W)$ represents the sensor's contour, a counterclockwise oriented closed curve that surrounds all the points that have been explored, Figure \ref{fig:contour}, and $\eta(\gamma,\bm{p})$ is its winding number with respect to $\bm{p}$.

Throughout the remainder of this Section, we extend the relation between the coverage measure and the topological degree so it comprehends more general scenarios.

 \subsection{ Coverage Measure for Points with Undefined Winding Numbers}
When the robot's pose and its visible set are well defined, the coverage measure of all the points in the environment during a mission can be uniquely determined. However, if we adopt the method proposed by \cite{oceans}, using relation  \eqref{eq:oceans_prop}, the coverage measure of a point $\bm{p} \in \gamma$ will be undefined considering the definition of winding numbers. 

\begin{figure}[h]
	\centering
	\includegraphics[scale=1.]{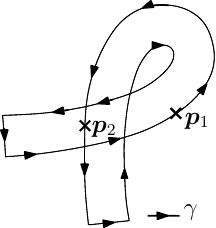}
	\caption{The coverage measure of point $\bm{p}_1$ is equal to $1$ and of point $\bm{p}_2$ is equal to $2$, but the winding number of $\gamma$ with respect to these points is undefined.}
	\label{fig:undef_wn}
\end{figure}

For example, in Figure \ref{fig:undef_wn}, point $\bm{p}_1 \in \gamma$ is the image by $\bm{f}$ of a point $(0,t) \in W$, for some $t \in [0,T]$. This point is inside the robot's visible area $\mathbb{V}(t)$ and according to the definition of the coverage measure on \eqref{eq:def_cm}, $c_m(\bm{p}_1) = 1$ even if $\eta(\gamma,\bm{p}_1)$ is undefined. In this context, to extend the validity of \eqref{eq:oceans_prop}, we define a bounded function $\overline{\eta}$ as the extension of the winding number function to the full domain $\bm{f}(W)$. For that, we consider the followingadapted from \cite{upper-semi}: 

\begin{definition}[Limit Superior]
 \label{def:uppersemicont}
Let $M$ be a metric space and $g$ a function from $M$ to $\mathbb{R}$. For any limit point $\bm{y} \in M$ the limit superior,  when it exists, is defined as:
$$\mathop{limsup}\limits_{\bm{p} \rightarrow \bm{y}} g(\bm{p}) = \lim\limits_{\epsilon\rightarrow 0} \ (sup  \{g(\bm{p}) \ | \ \bm{p} \in B(\bm{y},\epsilon)\backslash \{\bm{y}\}\})$$
\noindent where $B(\bm{y},\epsilon)$ denotes the ball within $M$, centered at $\bm{y}$, of radius $\epsilon$. 
\end{definition}

The sweep function $\bm{f}$ is a continuous map from a compact subset $W$ to $\mathbb{R}^2$, therefore $\bm{f}(W)\backslash \bm{f}(\partial W)$ is composed of a disjoint union of opens $V_i$, $i \in I$, for some index set $I$. All points of $\bm{f}(\partial W)$ are limits of some sequence of points $\bm{f}(\bm{y})$, with $\bm{y} \in \mathring{W}$. We can now state: 

\begin{lemma}
\label{lem:uppersemicont}
Consider a function $w: \ \mathop{\bigcup}\limits_{i\in I} V_i \rightarrow \Z$. Suppose that $w$ is bounded on $\mathop{\bigcup}\limits_{i\in I} V_i$ then there is an upper semi-continuous extension of $w$, $\overline{w}: \ \bm{f}(W) \rightarrow \Z$ defined as: 
$$
\overline{w}(\bm{p})=\left\{\begin{array}{ll}
w(\bm{p}) & \mbox{if $\bm{p}\in \mathop{\bigcup}\limits_{i\in I} V_i$} \\
\mathop{limsup}\limits_{\bm{p}' \in  \mathop{\bigcup}\limits_{i\in I} V_i \rightarrow \bm{p}} w(\bm{p}') & \mbox{ \text{otherwise}}
\end{array}\right.
$$
\end{lemma}

\begin{proof}
This is immediate: the limit sup exists since $w$ is bounded on $\bigcup\limits_{i\in I} V_i$, and the definition of $\overline{w}$ precisely imposes that $\overline{w}$ is upper semi-continuous. 
\end{proof}

Supposing that the number of connected components of $\bm{f}(W)\backslash \bm{f}(\partial W)$ is finite, as the winding number is constant on each component, this defines a bounded function $\eta$ that we can extend to the full domain $\bm{f}(W)$ by Lemma \ref{lem:uppersemicont} to obtain $\overline{\eta}$. Finally, if the condition expressed in \eqref{eq:det} is satisfied, we can say that for any $\bm{p} \in M$, 

\begin{equation}
	\label{eq:newdef_cm}
	\overline{\eta}(\gamma,\bm{p}) =  c_m(\bm{p})
\end{equation}

Considering Definition \ref{def:uppersemicont}, if $\bm{p} \in \gamma$, its coverage measure will be equal to the coverage measure of points on the open $V_i$ with the biggest winding number value for which $\bm{p}$ is a limit, as expected by the original definition on \eqref{eq:def_cm}. 

This new definition extends the applicability of the method but condition \eqref{eq:det} is still necessary for \eqref{eq:newdef_cm} to be true. Next section introduces new concepts to remove this constraint. 

\subsection{Coverage Measure for Points Swept Backwards}
Condition \eqref{eq:det} is necessary for \eqref{eq:newdef_cm} to be true. It ensures that the area surrounded by the sensor's contour $\gamma$ never shrinks during a mission and that $\gamma$ is indeed an enclosing curve for $\mathbb{A}_{\mathbb{E}}$. 

\begin{figure}[h]
		\centering
		\includegraphics[scale = 0.7]{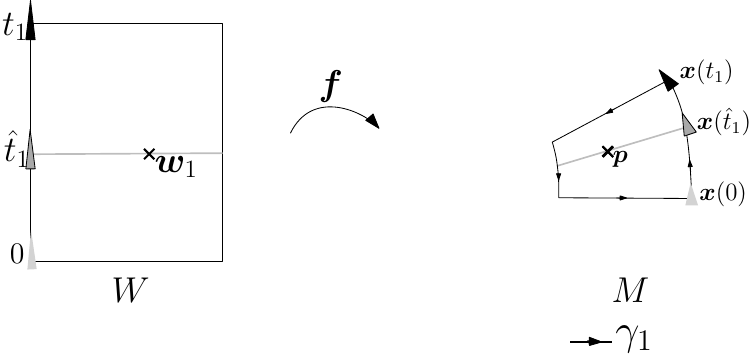}
		\caption{Mission during time interval $[0,t_1]$, point $\bm{p}$ is sensed for the first time at $\hat{t}_1$ and $c_m(\bm{p}) = 1$.}
		\label{fig:sweep1}
\end{figure}

\begin{figure}[h]
		\centering
		\includegraphics[scale = 0.7]{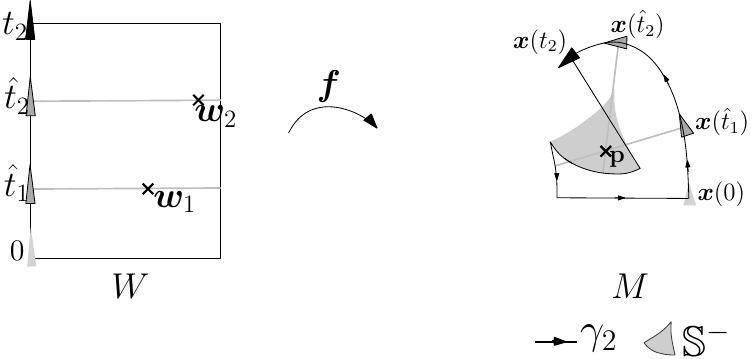}
		\caption{Condition established in Equation \eqref{eq:det} is not satisfied for all the points in $W$. At $t_2$, $c_m(\bm{p}) = 2$.}
		\label{fig:sweep2}
\end{figure}

\begin{figure}[h]
		\centering
		\includegraphics[scale = 0.7]{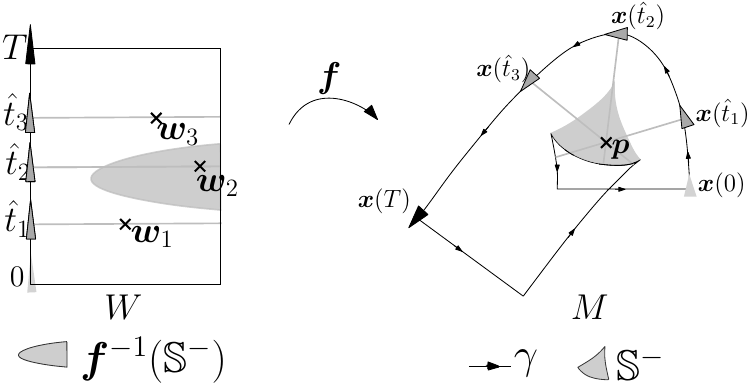}
		\caption{The mission ends at $T$ and the point $\bm{p}$ is sensed for the last time at $\hat{t}_3$, the final coverage measure of this point is $3$ although $\eta(\gamma,\bm{p}) = 1$.}
		\label{fig:sweep3}
\end{figure}

If condition \eqref{eq:det} is not satisfied, the inconsistency in the equality \eqref{eq:newdef_cm} is illustrated in Figures \ref{fig:sweep1},\ref{fig:sweep2} and \ref{fig:sweep3}. At the beginning of the mission, in Figure \ref{fig:sweep1}, the robot moves from its initial state $\bm{x}(0)$ to state $\bm{x}(t_1)$, $t_1 > 0$. During the interval $[0,t_1]$, condition \eqref{eq:det} is satisfied. Point $\bm{p} \in M$ is sensed for the first time at instant $\hat{t}_1 \in [0,t_1]$ and this occurrence is represented in the mission's Waterfall Space $W$ by point $\bm{w}_1$. The sensor's contour associated with this first part of the mission is the closed curve $\gamma_1 = \bm{f}(\partial([0, L] \times [0,t_1]))$ and $\eta(\gamma_1,\bm{p}) = sign(det(J_{\bm{f}}(\bm{w}_1))) = 1$ is indeed equal to the coverage measure of $\bm{p}$ at $t_1$. 

The mission continues as the robot advances to state $\bm{x}(t_2)$, $t_2 > t_1$ and point $\bm{p}$ is revisited at $\hat{t}_2$. For the time interval $[0,t_2]$, we have $\bm{f}^{-1}(\bm{p}) = \{ \bm{w}_1, \bm{w}_2\}$ and $\gamma_2 = \bm{f}(\partial([0,L] \times [0,t_2]))$ represents the sensor's contour. As illustrated in Figure \ref{fig:sweep2}, at $\hat{t}_2$, point $\bm{p}$ is swept in the opposite direction with respect to the robot's advancement movement. In this context, the Jacobian of function $\bm{f}$ at $\bm{w}_2$ is negative and $$\begin{aligned} \eta(\gamma_2,\bm{p}) & =  \sum\limits_{i=1}^2 sign(det(J_{\bm{f}}(\bm{w}_i)))  = 1 -1 = 0 \end{aligned}$$ although, according to \eqref{eq:def_cm}, $c_m(\bm{p}) = 2$ at $t_2$.

Exploration ends at state $\bm{x}(T)$, $T > t_2$ and the complete mission is represented in Figure \ref{fig:sweep3}. Point $\bm{p}$ is sensed for the third and last time at $\hat{t}_3$ and at the end of the mission $\bm{f}^{-1}(\bm{p}) = \{ \bm{w}_1, \bm{w}_2, \bm{w}_3\}$. At $\hat{t}_3$, point $\bm{p}$ is sensed by a forward movement of the sensor on the plane, therefore,  $$\begin{aligned} \eta(\gamma,\bm{p}) & =  \sum\limits_{i=1}^3 sign(det(J_{\bm{f}}(\bm{w}_i))) = 1 -1 + 1 = 1 \end{aligned}$$ but $c_m(\bm{p}) = 3$ is expected.

\begin{figure}[h]
	\centering
	\includegraphics[scale = 0.5]{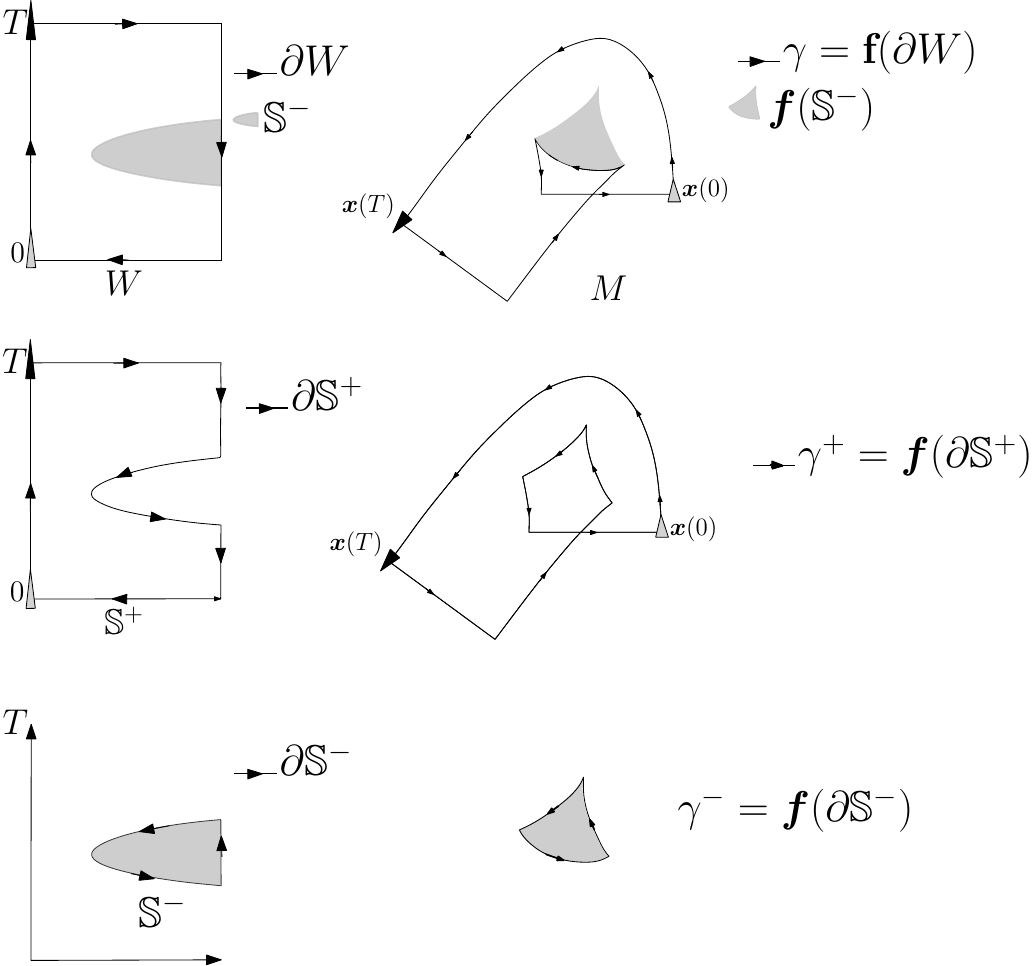}
	\caption{Decomposition of the Waterfall Space and $\gamma$ according to the sweeping direction.} 
	\label{fig:new_gammas}
\end{figure}

To address this problem, we can divide the Waterfall Space $W$ into two sets, $\mathbb{S}^+$ and  $\mathbb{S}^-$,  \begin{equation}  \label{eq:s_plus} \mathbb{S}^+ = \{ \bm{y} \in W |  det(J_{\bm{f}}(\bm{y})) > 0) \} \end{equation}  \begin{equation} \label{eq:s_minus}  \mathbb{S}^- =  \{ \bm{y}  \in W |  det(J_{\bm{f}}(\bm{y})) < 0) \}
\end{equation} We define two new positively oriented contours, $\gamma^+$ and $\gamma^-$ as the image by $\bm{f}$ of the boundaries of these sets, as illustrated in Figure \ref{fig:new_gammas}, 

\begin{equation}
	\label{eq:gamma_plus}
    \gamma^+ = \bm{f}(\partial \mathbb{S}^+)
\end{equation} \begin{equation}
    \gamma^- =  \bm{f}(\partial \mathbb{S}^-)
    \label{eq:gamma_minus}
    \end{equation}



For a regular value $\bm{p} \in M$ we will have $Ker \ (\bm{f} - \bm{p})  \subset \mathbb{S}^+\cup\mathbb{S}^-$, furthermore we can say that   \begin{equation}Ker \ (\bm{f} - \bm{p}) = Ker (\bm{f} - \bm{p})_{|\mathbb{S}^+} \cup Ker (\bm{f} - \bm{p})_{|\mathbb{S}^-} \end{equation} and we can rearrange \eqref{eq:def_cm}: \begin{equation}
  c_m(\bm{p}) =  \# Ker \ (\bm{f} - \bm{p})_{\mid \mathbb{S}^+} + \# Ker \ (\bm{f} - \bm{p})_{\mid \mathbb{S}^-} \end{equation}  
  \begin{equation}
	c_m(\bm{p}) =  \sum\limits_{\bm{w} \in {\bm{f}}_{\mid \mathbb{S}^+}^{-1}(\bm{p})} +1 \hspace{10pt} +  \sum\limits_{\bm{w} \in {\bm{f}}_{\mid \mathbb{S}^-}^{-1}(\bm{p})}+1 
  \end{equation}
  Considering the definitions of sets $\mathbb{S}^+$ and $\mathbb{S}^-$ on \eqref{eq:s_plus} and \eqref{eq:s_minus}, respectively,
 \small \begin{equation} {
  c_m(\bm{p}) =  \hspace{-10pt} \sum\limits_{\bm{w} \in {\bm{f}}_{\mid \mathbb{S}^+}^{-1}(\bm{p})} \hspace{-10pt}sign(det(J_{\bm{f}})(\bm{w}))  - \hspace{-15pt} \sum\limits_{\bm{w} \in {\bm{f}}_{\mid \mathbb{S}^-}^{-1}(\bm{p})} \hspace{-10pt} sign(det(J_{\bm{f}})(\bm{w}))  }\end{equation}\normalsize

Finally, considering Equations \eqref{eq:gamma_plus} and \eqref{eq:gamma_minus}, from \eqref{eq:topdeg_def} and Definition \ref{def:wn}, we obtain \begin{equation}
    \label{eq:last_cm_def}
    c_m(\bm{p}) = \overline{\eta}(\gamma^+,\bm{p}) + \overline{\eta}(\gamma^-,\bm{p})
\end{equation} for any regular point $\bm{p} \in M$. The extension to non-regular values $\bm{p}$ can be naturally done considering that $deg (\bm{f}, W, \bm{p})$  is locally constant on the connected components of $M \backslash \bm{f}(\partial W)$ \cite{Milnor1965}.

\subsection{Dealing with Uncertainties}
We now consider that the robot's pose can be uncertain, we keep the assumption that the sensor's model is exact. Since the visible set $\mathbb{V}$ depends on the robot's state, uncertainty is naturally propagated to the coverage measure. 

Let $\bm{x}^*$ be the robot's pose representing its position and orientation on the $\mathbb{R}^2$ plane during a mission. From now on, we assume that $\bm{x}^*$ is unknown and that instead, $\bm{x}^*$ belong to a set $[\bm{x}] \in \mathcal{P}(\mathbb{R} \rightarrow \mathbb{R}^3)$ of all the possible functions describing the robot's behavior. Modeling the state of a mobile robot by a set of possible solutions containing the ground truth is a common approach since they are usually nonholonomic systems. These are systems whose behavior can be modeled by differential equations and physical constraints, implying that if bounded uncertainties are introduced, they create a bounded disturbance around the real solution.

\begin{figure}[h]
	\centering
	\includegraphics[scale=1.]{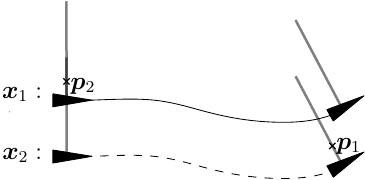}
	\caption{Point $\bm{p}_1$ is explored only if $\bm{x}^* = \bm{x}_2$ and point $\bm{p}_2$ is explored either if $\bm{x}^* =\bm{x}_1$ or $\bm{x}^* =\bm{x}_2$.}
	\label{fig:ex_trajs}
\end{figure}

The coverage measure $c_m(\bm{p})$ for a point $\bm{p} \in M$ can take different values for distinct functions $\bm{x} \in [\bm{x}]$. In this work, we propose a solution for computing the uncertain coverage measure based on interval analysis. For example, let us consider a set $[\bm{x}] = \{ \bm{x}_1,\bm{x}_2\}$ with two possible solutions as illustrated in Figure \ref{fig:ex_trajs}. The coverage measure of $\bm{p}_1$ can either be $0$ or $1$. In this case, we want its coverage measure to be represented by an interval $[0,1]$ containing all the possible solutions. For point $\bm{p}_2$, its coverage measure is always equal to $1$. Therefore, we represent its coverage measure by the singleton $[1,1]$.




We adopt the notation ${c_m}_{|\bm{x}}(\bm{p})$ for representing the coverage measure of a point $\bm{p} \in M$ for a given $\bm{x}$. We are interested in estimating $[c_m](\bm{p}) \in \mathbb{I}\mathbb{Z}$, an interval of relative integers such that
 \begin{equation}
	\forall \bm{x} \in [\bm{x}] \ , \ {c_m}_{|\bm{x}}(\bm{p}) \in [c_m](\bm{p}). 
\end{equation}

From each $\bm{x}  \in [\bm{x}]$, we can generate a different  $\gamma$, a possible sensor's contour for the mission. We define $[\gamma] \in \mathcal{P}( S^1 \rightarrow \mathbb{R}^2)$ as the set of all possible $\gamma$. To simplify the definitions, first we consider a point $\bm{p} \in M $ such that $det(J_{\bm{f}}(\bm{w})) > 0$ for all $\bm{w} \in \bm{f}^{-1}(\bm{p})$. In this case, according to   \eqref{eq:newdef_cm}, we can obtain the coverage measure through the computation of the winding number of the sensor's contour. Therefore, we want to determine $[\overline{\eta}]([\gamma],.) \in \mathbb{I}\mathbb{Z}$ such that 
\begin{equation}
	\forall \gamma \in [\gamma] \ , \  \overline{\eta}(\gamma,.) \in [\overline{\eta}]([\gamma],.).
\end{equation} and we can define the uncertain coverage measure of $\bm{p}$ as 
\begin{equation}
	[c_m](\bm{p})  = [\overline{\eta}]([\gamma],\bm{p}).
\end{equation}
A generalization of the results stated in the remaining of this Section for all the points in the plane can be easily obtained considering a decomposition of cycles $\gamma \in [\gamma]$ in $\gamma^+$ and $\gamma^-$ as proposed in \eqref{eq:last_cm_def}.

\section{Computing the Coverage Measure}
\label{sec:compute}
We are interested in determining the coverage measure of all the points inside an area of interest. Thus, we developed an algorithm, that is presented in this Section, for computing the extended winding number function $\overline{\eta}$ for a cycle $\gamma: S_1 \rightarrow \mathbb{R}^2$ with respect to all the points inside a subset of $\mathbb{R}^2$. We also present its extension for dealing with an uncertain cycle $[\gamma]$. 


\subsection{Computing the Extended Winding Number of $\gamma$ }

Let $\mathbb{W}_i$ be a winding set associated with a cycle $\gamma$, defined for a natural number $i$, by definition \begin{equation}
	\label{eq:winset}
	\mathbb{W}_i := \{ \bm{p} \in \mathbb{R}^2 | \eta(\gamma,\bm{p}) \geq i \}
\end{equation}

\begin{figure}[h]
		\centering
		\includegraphics[scale = 0.9]{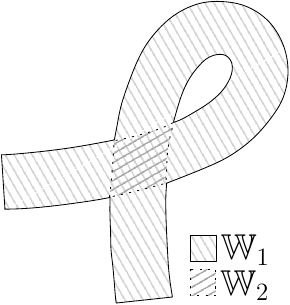}
		\caption{Winding sets $\mathbb{W}_1$ and $\mathbb{W}_2$ associated with the curve $\gamma$ illustrated in Figure \ref{fig:contour}.}
	\label{fig:windingset_ex}
\end{figure}

There are, for example, two non-empty winding sets associated with the curve $\gamma$ of Figure \ref{fig:contour}, $\mathbb{W}_1$ and $\mathbb{W}_2$ represented in Figure \ref{fig:windingset_ex}. As demonstrated in \cite{mcintyre}, the winding number $\eta(\gamma,\bm{p})$ of any point  $\bm{p} \in \mathbb{R}^2 \setminus \gamma$ can be calculated using the winding sets of $\gamma$, \begin{equation}
		\label{eq:mc_theorem}
		\eta(\gamma,\bm{p}) = \sum_{i > 0} \chi_{\mathbb{W}_i}(\bm{p})
	\end{equation} where $\chi_{\mathbb{W}_i}$ is the characteristic function for the winding set $\mathbb{W}_i$. Equations \eqref{eq:winset} and \eqref{eq:mc_theorem} are still valid if $\eta$ is replaced by its extension $\overline{\eta}$. 

 The algorithm starts by computing all the non-empty winding sets $\mathbb{W}_i$, for $i \in \mathbb{N}$, associated with the sensor's contour $\gamma$, through a combinatorial approach.
For that, we consider that a self-intersection or vertex of $\gamma$ is determined by two parameters $t_0 ,t_1 \in S_1$, $t_0 \neq t_1$ and that it is a point $\bm{p}$ such that $\bm{p} = \gamma(t_0) = \gamma(t_1)$. The multiplicity of such a self-intersection is the number, finite or infinite, of distinct $t \in S_1$ such that $\bm{p} = \gamma(t)$ minus one. Then, we make the following assumptions, similar to those of \cite{alexander}, so that the winding number of a point can be easily obtained using \eqref{eq:mc_theorem}:
\begin{itemize}
    \item $\gamma$ has a finite number of self-intersections, each one of them with multiplicity one.
    \item in addition, we assume the two tangent vectors to $\gamma$ at each vertex to be linearly independent. 
\end{itemize} 

\begin{figure}[h]
  \centering
  \subfigure[]{\includegraphics[scale = 1.]{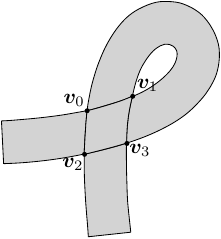}}\quad
  \subfigure[]{\includegraphics[scale=1.]{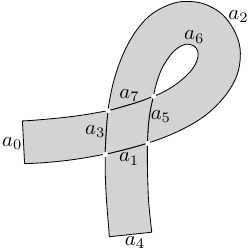}}\quad
  \subfigure[]{\includegraphics[scale=1.]{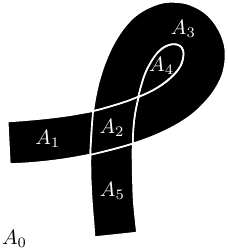}}
  \caption{(a): $CW(\gamma)$ has four 0-cells $\{\bm{v}_0,\bm{v}_1,\bm{v}_2,\bm{v}_3\}$ ; (b): $CW(\gamma)$ has eight 1-cells, connected components of $\gamma \setminus \{\bm{v}_0,\bm{v}_1,\bm{v}_2,\bm{v}_3\}$; (c): The plane is divided into six 2-cells, five compacts (from $A_1$ to $A_5$) and one extending to the infinity ($A_0$).}
	\label{fig:cw}
\end{figure}


	Such a cycle divides  $\mathbb{R}^2 \setminus \gamma$ into a finite number of connected open regions, one of which is not compact. Each one of these regions can be seen as a $2-cell$ of the CW-complex $C(\gamma)$, constructed from the cycle $\gamma$. To be fully formal, we would need to use the fact that $\gamma$ determines a cell decomposition of the one-point compactification of the plane, homeomorphic to the 2-sphere $S_2$, Figure \ref{fig:cw}. The 0-cells of $C(\gamma)$ are self-intersections of $\gamma$, and the 1-cells are parts of the curve separating the 2-cells, connected components of $\gamma$ minus its self-intersections. 

	
	
	

Since all open 2-cells are homotopy equivalent to a point within that cell and considering the degree axioms presented in Definition \ref{def:deg}, we can conclude that all the points within the same open 2-cell of $C(\gamma)$ have the same winding number with respect to $\gamma$. In this context, a correct and coherent numbering of the 2-cells is enough for determining the winding number value of all the points in the plane.

\begin{figure}[h]
  \centering
  \subfigure[]{\includegraphics[scale = 0.8]{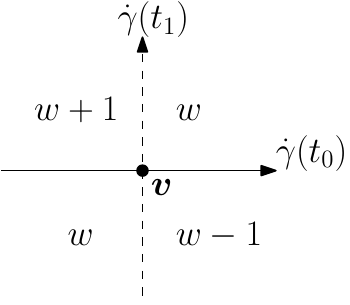}}\quad
  \subfigure[]{\includegraphics[scale=0.8]{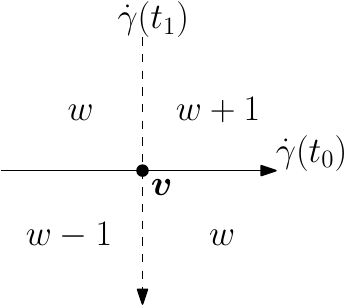}}
  \caption{Alexander numbering with $w \in \mathbb{Z}$: (a): $\dot{\gamma}(t_1)$ crosses $\dot{\gamma}(t_0)$ from right to left; (b): $\dot{\gamma}(t_1)$ crosses $\dot{\gamma}(t_0)$ from left to right.}
	\label{fig:alex1}
\end{figure}


For this purpose, we can use a combinatorial rule proposed by Möbius in 1865 \cite{mobius}. The rule says that two contiguous regions that are separated by a 1-cell are numbered with a value that must differ by exactly 1. The winding number of the region on the left is greater, considering the curve's orientation. This method leads to a unique numbering of the space considering that the winding number in the non-compact region, to whom we will be referring as $A_{0}$, is known and equal to $0$ for all of its points. This is true because since $A_{0}$ is not bounded by $\bm{f}(\partial W)$, differently from the other 2-cells of $C(\gamma)$, we know that $A_{0} \subseteq \mathbb{R}^2 \backslash \bm{f}(W)$. This implies, from Definition \ref{def:wn}, that for any $\bm{p} \in A_{0}$, $\eta(\gamma,\bm{p}) = 0$.

As a direct application of Möbius rules, a method proposed by Alexander \cite{alexander} allows a coherent numbering of the regions only through an analysis of the tangent vectors to the curve on its self-intersections. Let $\bm{v}$ be a vertice of $\gamma$ represented by the pair $(t_0,t_1)$. Considering the assumptions adopted for $\gamma$, a self-intersection $\bm{v}$ will divide the plane into four regions. 
There are only two rules for numbering these four regions, according to whether $\dot{\gamma}(t_1)$ goes from the right to the left or the left to the right with respect to $\dot{\gamma}(t_0)$, as illustrated in Figure \ref{fig:alex1}. 

\begin{figure}[h]
		\centering
		\includegraphics[scale = 0.9]{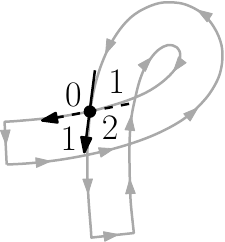}
		\caption{Numbering of regions according to Alexander around $\bm{v}_0$.}
		\label{fig:sub1alex2}
\end{figure}
	\begin{figure}[h]
		\centering
		\includegraphics[scale = 0.9]{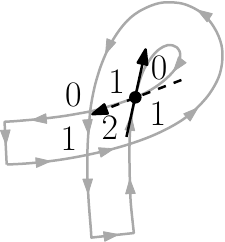}
		\caption{Numbering of regions according to Alexander around $\bm{v}_1$.}
		\label{fig:sub2alex2}
\end{figure}
	\begin{figure}[h]
		\centering
		\includegraphics[scale = 0.9]{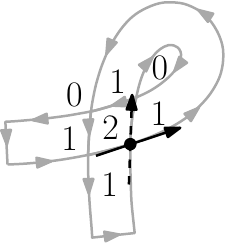}
		\caption{Numbering of regions according to Alexander around $\bm{v}_2$.}
		\label{fig:sub2alex3}
\end{figure}

In Figures \ref{fig:sub1alex2},\ref{fig:sub2alex2} and \ref{fig:sub2alex3} we consecutively apply the Alexander numbering rules to the example considered previously. We start by numbering regions around $\bm{v}_0$, Figure \ref{fig:sub1alex2}. We assume that $A_0$ has a winding number value of $0$ and that the later self-intersection, represented by the dashed line, crosses the previous one from left to the right. The same is done around vertices $\bm{v}_1$ and $\bm{v}_2$ at Figure \ref{fig:sub2alex2} and \ref{fig:sub2alex3}, respectively, resulting in a complete characterization of the plane in terms of winding number values.

Once a numbering is obtained for all the regions according to Alexander's rules, we can construct the winding sets  $W_i$ of $\gamma$, for $i \in \mathbb{N}$, as the closure of the union of the regions with a number greater than or equal to $i$ \cite{mcintyre}. Then, the winding number for a point can be easily computed using \eqref{eq:mc_theorem}.

\subsection{Computing the Extended Winding Number of $[\gamma]$ }
If the sensor's contour $\gamma$ is uncertain, the winding sets associated with the mission will also be uncertain. An uncertain set can be represented as a thick set, the following definition was proposed in \cite{thicksets}. 

\begin{figure}[h]
	\centering
	\includegraphics[scale=0.8]{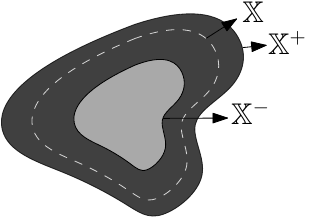}
	\caption{Representation of thick sets.}
	\label{fig:thickset}
\end{figure}
\vspace{0.5cm}
\begin{definition}
	\label{def:thickset}
	We denote $\llbracket \mathbb{X} \rrbracket \in \mathbb{I}\mathcal{P}(\mathbb{R}^n)$ a thick set of $\mathbb{R}^n$ if there are two subsets of $\mathbb{R}^n$ called the lower bound $\mathbb{X}^-$ and the upper bound $\mathbb{X}^+$ such that \begin{equation}
		\begin{aligned}
			\llbracket \mathbb{X} \rrbracket & = [\mathbb{X}^-,\mathbb{X}^+] \\
			& = \{\mathbb{X} \in \mathcal{P}(\mathbb{R}^n) \ | \ \mathbb{X}^- \subseteq \mathbb{X} \subseteq \mathbb{X}^+ \}
		\end{aligned}
	\end{equation}
	A thickset partitions the environment into three zones, the clear zone $\mathbb{X}^-$, the penumbra $\mathbb{X}^+ \backslash \mathbb{X}^-$ (both illustrated in Figure \ref{fig:thickset}) and the dark zone $\mathbb{R}^n \backslash \mathbb{X}^+$.
\end{definition}
Let $\mathbb{W}_i^{\gamma}$, with $i \in \mathbb{N}$, be a winding set associated with a cycle $\gamma$. To the set $[\gamma]$ of all the possible sensor's contour we associate $\llbracket \mathbb{W}_i \rrbracket = [\mathbb{W}^-_i,\mathbb{W}^+_i]$, such that,

\begin{equation}
	\mathbb{W}^-_i = \bigcap_{\gamma \in [\gamma]} {\mathbb{W}^{\gamma}_i}
\end{equation}
\begin{equation}
	\mathbb{W}^+_i = \bigcup_{\gamma \in [\gamma]} {\mathbb{W}^{\gamma}_i}
\end{equation}

\begin{figure}[h]
  \centering
  \subfigure[]{\includegraphics[scale=0.7]{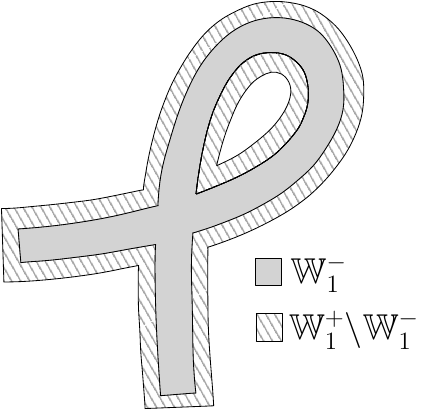}}\quad\hspace{0.8cm}
  \subfigure[]{\includegraphics[scale=0.7]{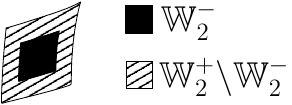}}
  \caption{(a):$\llbracket \mathbb{W}_1 \rrbracket$ ; (b): $\llbracket \mathbb{W}_2 \rrbracket$ .}
	\label{fig:gamma_tube_1}
\end{figure}

In the exploration context, the clear zone of $\llbracket \mathbb{W}_i \rrbracket$, represented by $\mathbb{W}^-_i$, translates as a set of points that were certainly explored at least $i$ times. Analogously, the dark zone $\mathbb{R}^2 \backslash \mathbb{W}^+_i$ is a set of points that have a coverage measure smaller than $i$, independently of which of the functions in $[\bm{x}]$ is the ground truth. The penumbra $\mathbb{W}^+_i \backslash \mathbb{W}^-_i$ is a set of points whose coverage measure is equal to $i$ for some $\gamma \in [\gamma]$. 

\begin{figure}[h]
		\centering
		\includegraphics[scale = 0.8]{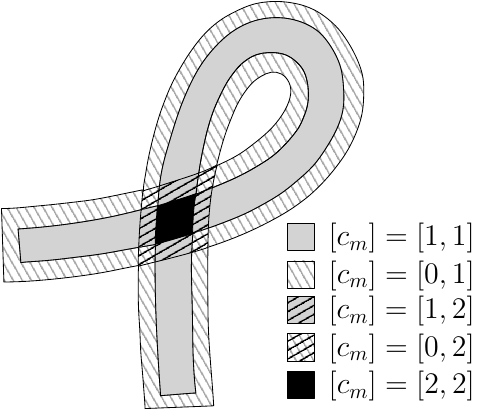}
	   \caption{Coverage measure considering the uncertain winding sets associated with $[\gamma]$.}
	   \label{fig:gamma_tube_2}
\end{figure}

We redefine the characteristic function to deal with thick sets on the plane, we have $[\chi]:\mathbb{R}^2 \rightarrow \mathbb{I}\mathbb{N}_0$ and \begin{equation}
	\label{eq:chi_thick}
	[\chi]_{\llbracket \mathbb{W}_i \rrbracket}(\bm{p}) = \begin{cases}
		[1,1], & \text{if}\ \bm{p} \in \mathbb{W}_i^-,\\
		[0,1], & \text{if}\ \bm{p} \in \mathbb{W}_i^+ \backslash \mathbb{W}_i^-,\\
		[0,0], & \text{otherwise}
	\end{cases}
\end{equation} Then, we have \begin{equation}
	\label{eq:new_etacalc}
	[\overline{\eta}]([\gamma],\bm{p}) = \sum_{i > 0} \chi_{\llbracket \mathbb{W}_i \rrbracket}(\bm{p})
\end{equation} In Figure \ref{fig:gamma_tube_1} we have an illustration of thick sets $\llbracket \mathbb{W}_1 \rrbracket$ and $\llbracket \mathbb{W}_2 \rrbracket$ for the example considered through out this paper and in Figure \ref{fig:gamma_tube_2} the resultant coverage measure considering these sets.

This defines the notion of uncertain winding number (and uncertain coverage measure). Under some assumptions, given below, that are realistic for applications, we need only a slightly generalized Alexander rule to efficiently compute the uncertain coverage measure. 

As in \cite{Rohou}, we will suppose that $[\bm{x}]$ is given by two time-varying sets: an outer approximation of the set of the robot's pose, $[\bm{s}](t)$, at time $t$, in the plane, and $[\bm{v}](t)$, an outer-approximation of the set of linear velocities of the robot, at time $t$, in the plane. Hence:
$$\begin{array}{lrcl}
[\bm{s}]: & \R & \rightarrow & \R^2 \\
{[} \bm{v} {]}: & \R & \rightarrow & \R^2
\end{array}$$

Consider the following notion of uncertain self-intersection. These are points $\bm{p}$ in the plane such that $\bm{p} \in [\bm{s}](t_1) \cap [\bm{s}](t_2)$ for some $t_1 < t_2$. The set of pairs of such times $t_1$, $t_2$, for a given $\bm{p}$, is denoted by $T_x$. Supposing that for all $\bm{p}$ uncertain self-intersection, for all $(t_1,t_2)\in T_x$, for all $v_1 \in [\bm{v}](t_1)$, $v_2 \in [\bm{v}](t_2)$, $v_1$ is not colinear with $v_2$ (or $v_1$ and $v_2$ are transverse to each other), we get the following uncertain Alexander rules: 


\begin{figure}[h]
  \centering
  \subfigure[]{\includegraphics[scale = 0.9]{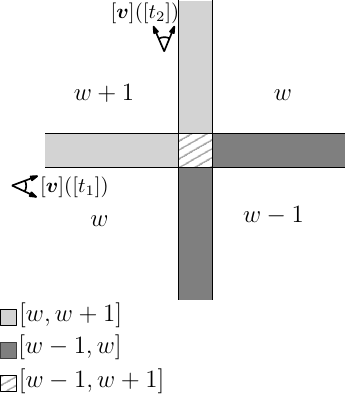}}\quad\hspace{0.8cm}
  \subfigure[]{\includegraphics[scale = 0.9]{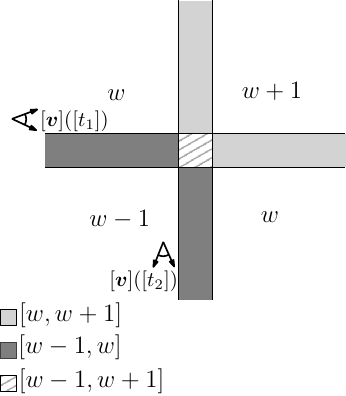}}
  \caption{Uncertain Alexander numbering with $w \in \mathbb{Z}$: (a): $[\bm{v}](t_2)$ comes from the right; (b): $[\bm{v}](t_2)$ comes from the left.} 
	\label{fig:alex4}
\end{figure}

\begin{subsection}{Implementation}
    The method above was numerically implemented using the Codac library \cite{codac}.\footnote{The code is available on GitHub  \href{https://github.com/marialuizacvianna/extended_winding}{github.com/marialuizacvianna/extended\_winding} .} We consider that we have on the input of the algorithm a  well defined function or a tube describing the robot's pose $\bm{x}$, speed $\dot{\bm{x}}$ and acceleration $\ddot{\bm{x}}$. From these inputs, the sensor's contour $\gamma$ is obtained through a concatenation of $\bm{x} = \bm{f}(0,[0,T])$ with $\bm{x}_{aux1} = \bm{f}([0,L],T)$, $\bm{x}_R = \bm{f}(L,[0,T])$ and $\bm{x}_{aux2} = \bm{f}([0,L],0)$, as illustrated in Figure \ref{fig:wf_mosaic_onesonar_left} and we have $$
        \gamma = \bm{x} * \bm{x}_{aux1} * \bm{x}^{-1}_R * \bm{x}^{-1}_{aux2}
    $$ where $\bm{x}^{-1}_R(t) = \bm{x}_R(T-t)$ and $\bm{x}^{-1}_{aux2}(t) = \bm{x}_{aux2}(T-t)$. We parameterize $\gamma$ with $\tau \in [0,1]$ that is not a time representation. The speed vector along $\gamma$ can be computed using $\dot{\bm{x}}$ and $\ddot{\bm{x}}$.

    The next step in the algorithm is to compute the set of time pairs $\mathbb{T}$ that represent the self-intersections of $\gamma$.$$
        \mathbb{T} = \{ (\tau_1,\tau_2) \in [0,1]^2 | \tau_1 < \tau_2 \text{ and } \gamma(\tau_1) = \gamma(\tau_2) \}
    $$ This set can be obtained with the algorithm presented in \cite{aubry} available in \cite{codac}. For the  example considered throughout this paper, first presented in Figure \ref{fig:line_sweep_ae}, we obtain the following set of self-intersections $$ \mathbb{T} = \{ (\tau_1,\tau_4), (\tau_2,\tau_5), (\tau_6,\tau_7), (\tau_3,\tau_8) \} $$ where $0 \leq \tau_1 <  \tau_2 < \hdots < \tau_8  \leq 1$. These pairs correspond to the vertices illustrated in Figure \ref{fig:cw}:  $\bm{v}_0 = \gamma(\tau_3) = \gamma(\tau_8)$, $\bm{v}_1 = \gamma(\tau_6) = \gamma(\tau_7)$, $\bm{v}_2 = \gamma(\tau_0) = \gamma(\tau_1)$ and $\bm{v}_3 = \gamma(\tau_2) = \gamma(\tau_5)$. Then, the set of 1-cells of $\gamma$ can be defined as $$\mathbb{E} =  \{ a_0, a_1, a_2, a_3, a_4, a_5, a_6, a_7 \}$$ where $\partial a_i = \gamma(\tau_{i+1}) - \gamma(\tau_i)$, for $i = 1, \hdots \#\mathbb{E} - 1$ and $\partial a_0 = \gamma(\tau_{1}) - \gamma(\tau_{\#\mathbb{E}})$.

    Determining if a vector $\bm{a}$ crosses another vector $\bm{b}$ from the right to the left can be mathematically translated by the cross product $\bm{a} \times \bm{b}$ being positive. In this case, to each of the vertices represented by a pair $(\tau_i,\tau_j) \in \mathbb{T}$ we associate an update value $u \in \{ -1, +1\}$ that determines if $\dot{\gamma_j}$ crosses $\dot{\partial(U_i)}$ from the right to the left $u = -1$ or the left to the right $u = +1$. 

 We use the update value of each edge's initial vertex and the combinatorial method presented in this Section for defining a winding number value for the area on its right and left sides. Finally, the winding sets can be easily obtained knowing that $\partial \mathbb{W}_i$ is a concatenation of the edges in $\mathbb{E}$ for which the value on the area on its left side is equal or greater than $i$.  

 We choose to represent sets using interval arithmetic and we rely on interval analysis tools \cite{Moore}, such as separators and a Set Inversion Via Interval Analysis (SIVIA) algorithm \cite{sivia}, for classifying, in terms of their coverage measure, all the points inside an area of interest. The set inversion algorithm bisects the environment, up to a precision that is chosen by the user, such that the plane is divided into boxes that do not intersect $\gamma^+$ and $\gamma^-$. The advantage of this method is that it is known, from the properties of the topological degree, that all the points that belong to a set in the plane that does not intersect the considered cycles will have the same winding number value. Therefore, this method limits the number of computations that have to be done to determine the winding number for all the points inside an area. For boxes $[\bm{b}] \in \mathbb{I}\mathbb{R}^2$ for which  $[\bm{b}] \cap \gamma^+ \neq \emptyset$ or $[\bm{b}] \cap \gamma^- \neq \emptyset$ is true, an uncertain winding number value will be computed. For that, we use the following adaptation of the characteristic function for thick sets to deal with sets of $\mathbb{R}^2$ on the input: 
$[\chi]:\mathcal{P}(\mathbb{R}^2) \rightarrow \mathbb{I}\mathbb{N}_0$, \begin{equation}
	\label{eq:chi_thick}
	[\chi]_{\llbracket \mathbb{W}_i \rrbracket}([\bm{b}]) = \begin{cases}
		[1,1], & \text{if for all}\ \bm{p} \in [\bm{b}] \text{ , } \bm{p} \in \mathbb{W}_i^-,\\
		[0,1], & \text{if}\ \exists \ \bm{p} \in [\bm{b}] \text{ , } \bm{p} \in \mathbb{W}_i^+ \backslash \mathbb{W}_i^-,\\
		[0,0], & \text{otherwise}
	\end{cases}
\end{equation}


\end{subsection} 



%

\section{Experiments}
\label{sec:exp_validation}
\begin{figure}[h]
    \centering
    \includegraphics[scale = 0.25]{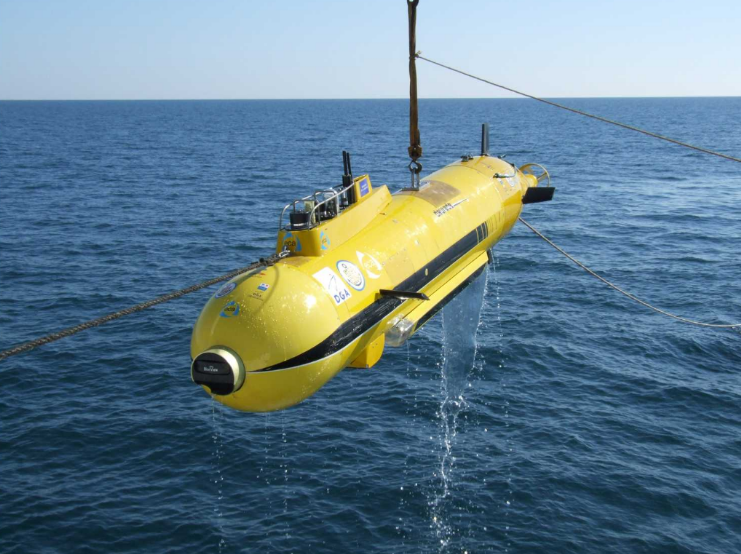}
    \caption{The AUV Daurade.}
    \label{fig:9}
\end{figure}

We apply the method presented in this paper on a dataset acquired during a mission performed by the AUV daurade, Figure \ref{fig:9}, on November 2015. This robot was built by ECA robotics and used by Direction Général de l'Armement - Techniques Navales (DGA - TN) and by the Service Hydrographique et Océanogrpahique de la Marine (SHOM). The mission took place in the Road-Sted of Brest (Britanny, France), it consists of a 45 minutes survey path.

Daurade explores using two side-scan sonars, one that explores its right side and the other its left side. The visible area of both sensors can be individually modeled as a line-sweep sensor on the plane. Assuming a configuration in which there is no visibility gap and no overlap between the range of visibility of the two sensors, the whole can be represented as a line-sweep sensor. 

\begin{figure}[h]
  \centering
  \includegraphics[ trim={2cm 2cm 0 0}, clip, scale = 0.35]{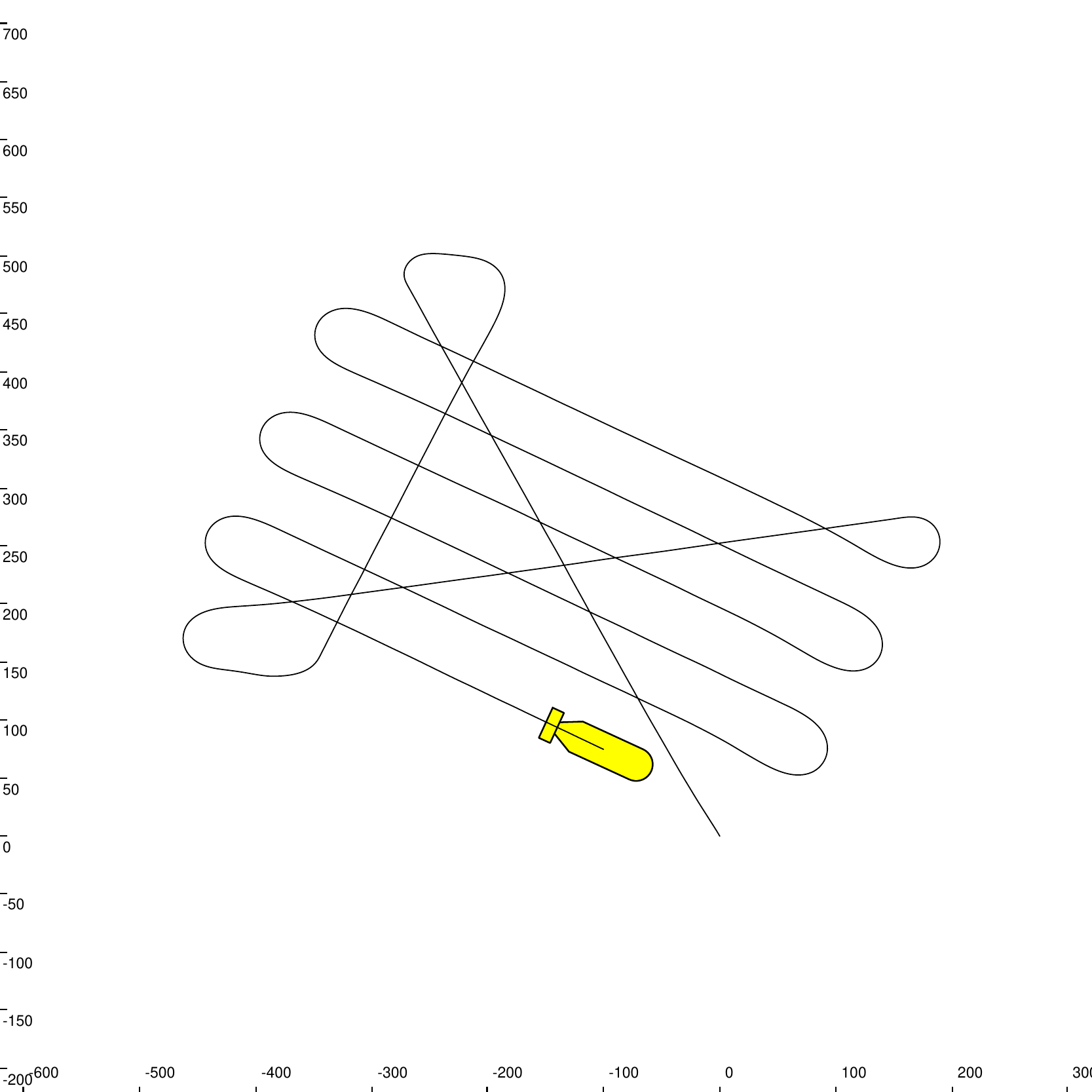}
    \caption{Estimated robot's trajectory $\tilde{\bm{x}}$ without incertitude. The robot is represented at its final pose at the end of the mission.}
  \label{fig:10sub1}
\end{figure}
\begin{figure}[h]
  \centering
  \includegraphics[trim={2cm 2cm 0 0}, clip, scale = 0.35]{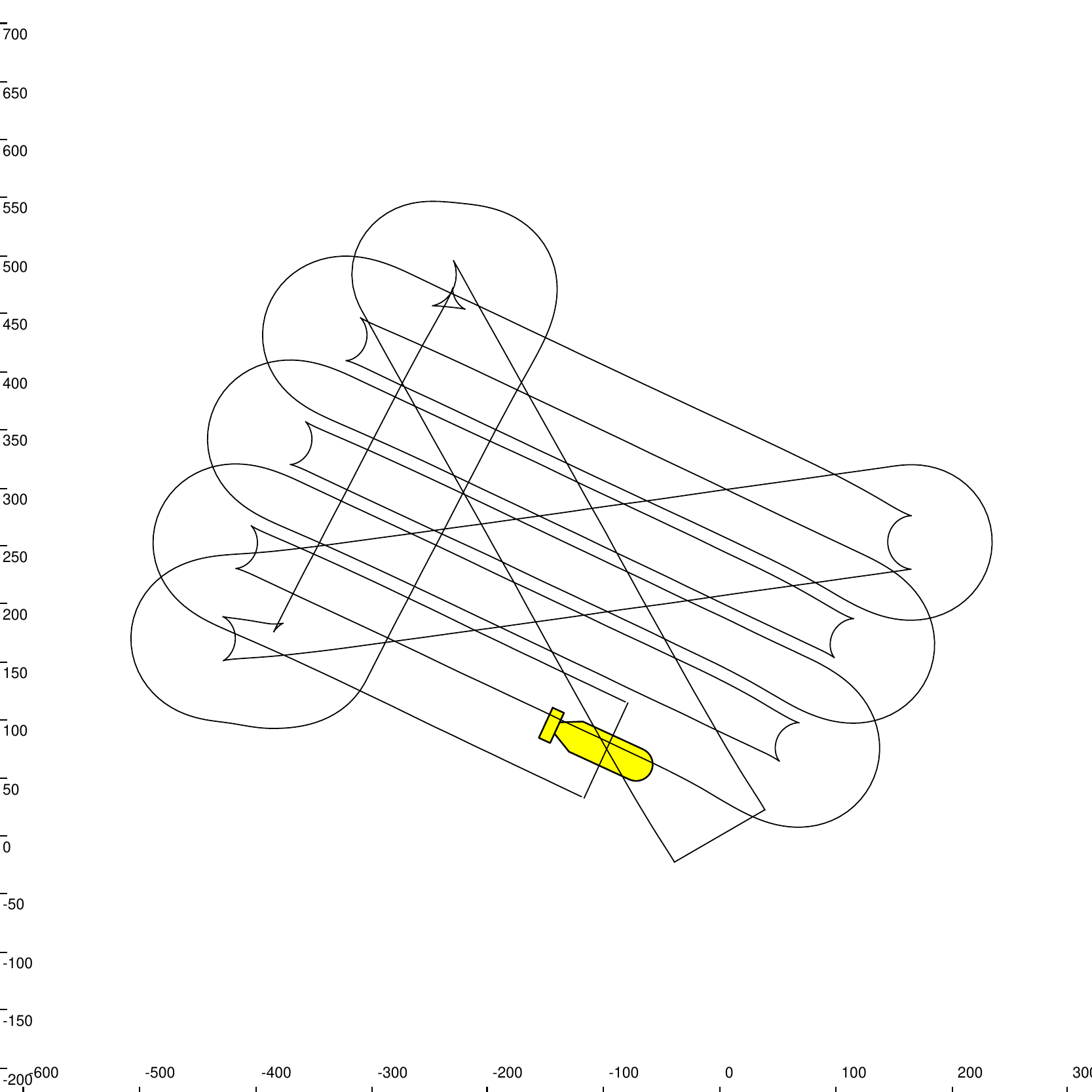}
  \caption{The sensor's contour $\tilde{\gamma}$ for the mission.}
  \label{fig:10sub2}
\end{figure}

\begin{figure}[h]
  \centering
  \subfigure[]{\includegraphics[ trim={2cm 2cm 0 0}, clip, scale = 0.35]{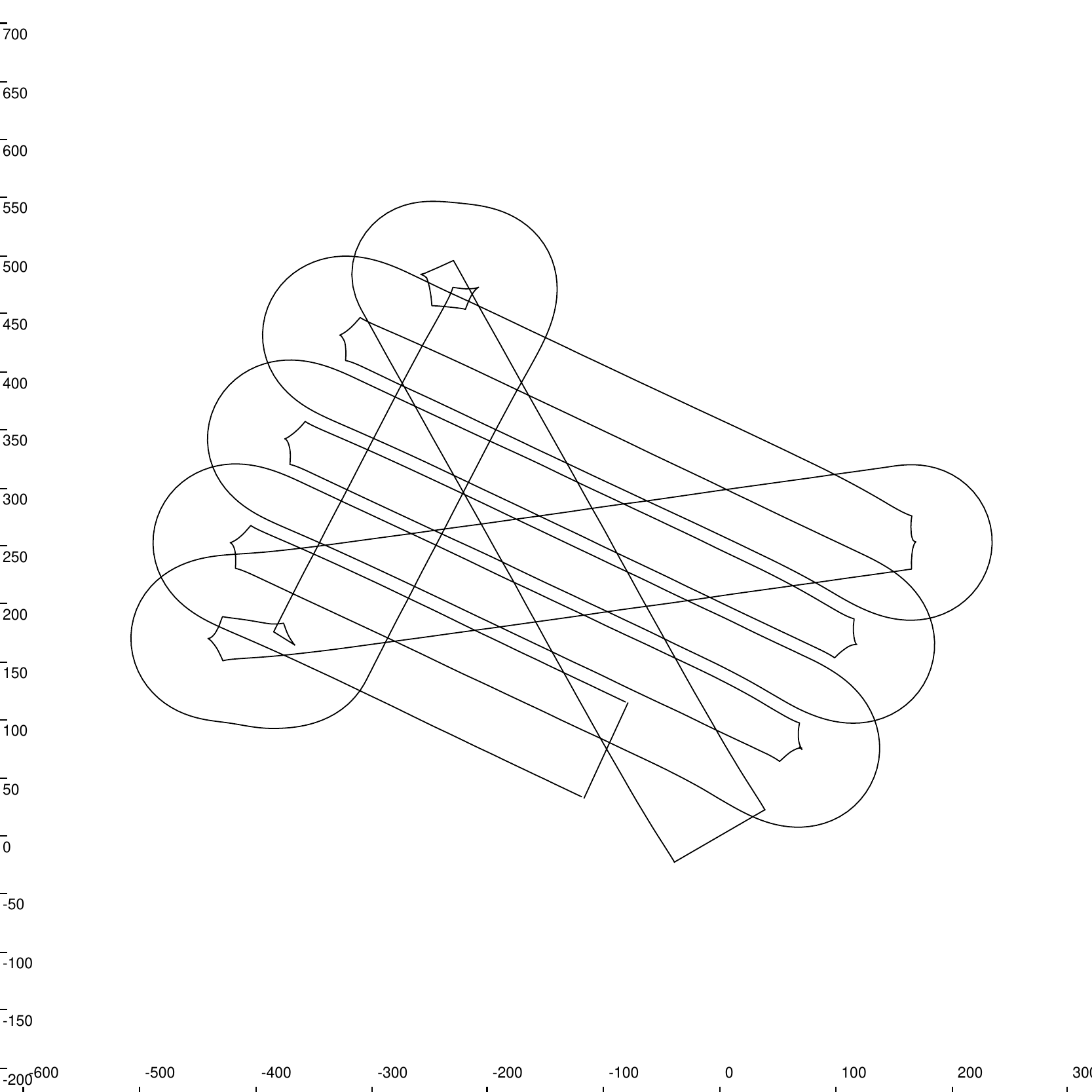}}\quad
  \subfigure[]{\includegraphics[trim={2cm 2cm 0 0}, clip, scale = 0.35]{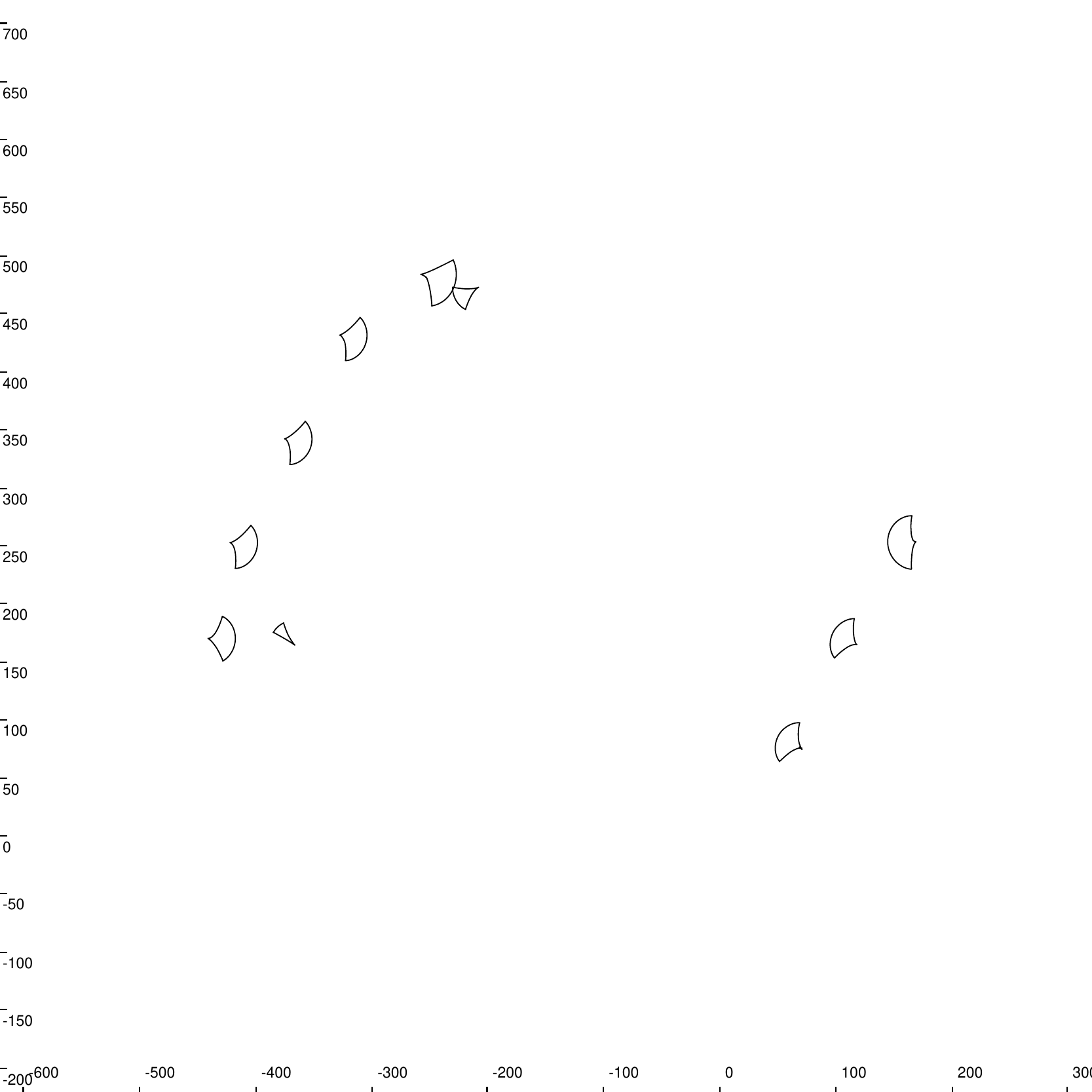}}
  \caption{(a): $\gamma^+$; (b): $\gamma^-$.}
    \label{fig:11}
\end{figure}

The robot's pose underwater is estimated by the integration of data acquired by an Inertial Measurement Unit (IMU) coupled with a Doppler Velocity Logger (DVL) and a pressure sensor, for depth estimation. Initially, we assume that this estimation $\tilde{\bm{x}}$ is exact, as illustrated in Figure \ref{fig:10sub1}, and that the robot maintains a constant depth during the mission, resulting in the sensor's contour $\tilde{\gamma}$ presented in Figure \ref{fig:10sub2}. Figure \ref{fig:11} displays the separation of $\tilde{\gamma}$ into $\tilde{\gamma}^+$ and  $\tilde{\gamma}^-$. 

\begin{figure}[h]
\centering
  \includegraphics[scale = 0.5]{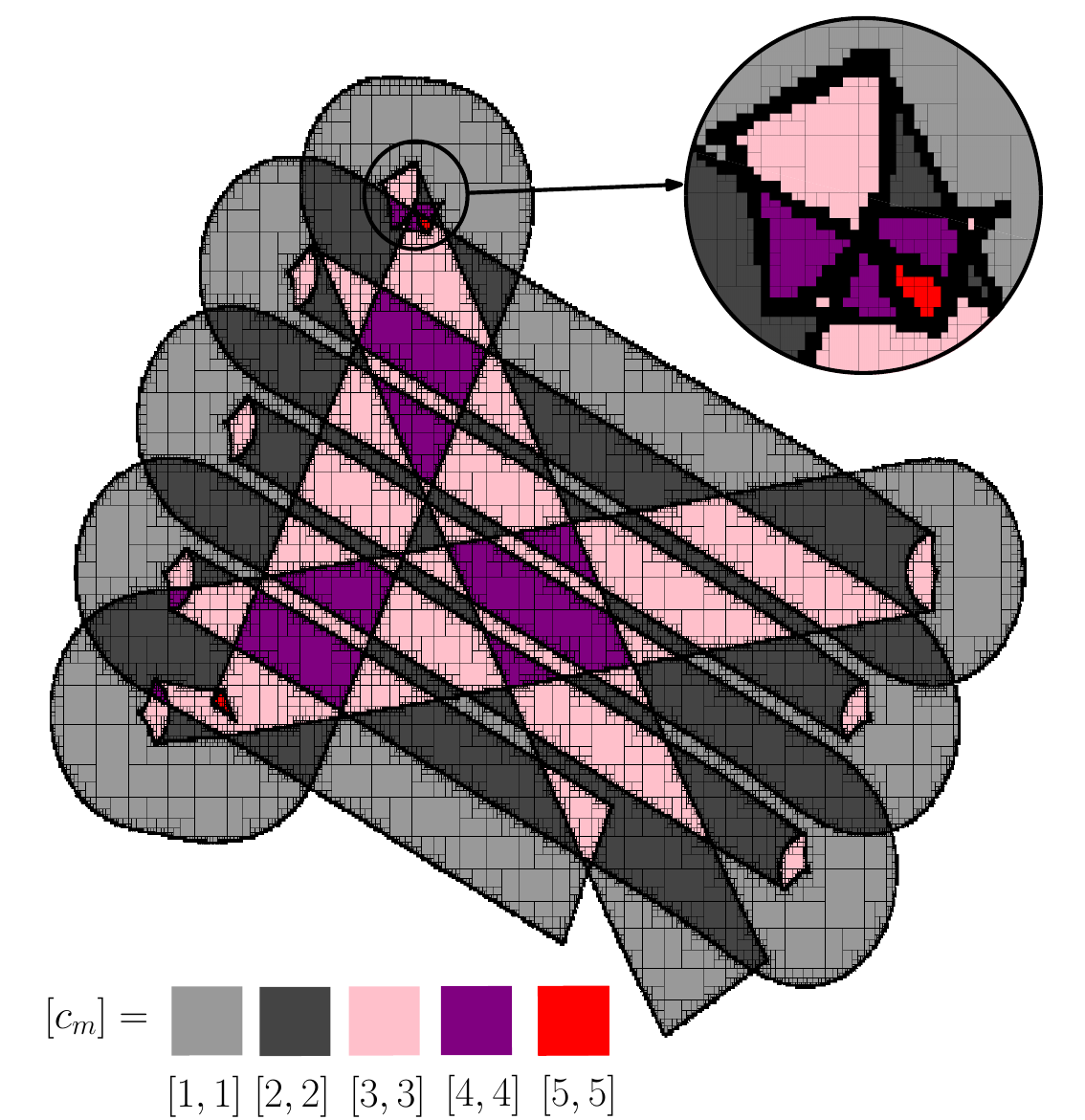}
\caption{Result of the SIVIA algorithm for the classification of the explored area. Boxes in black have an uncertain coverage measure value.}
\label{fig:12}
\end{figure}


\begin{figure}[h]
\centering
\includegraphics[scale = 0.4]{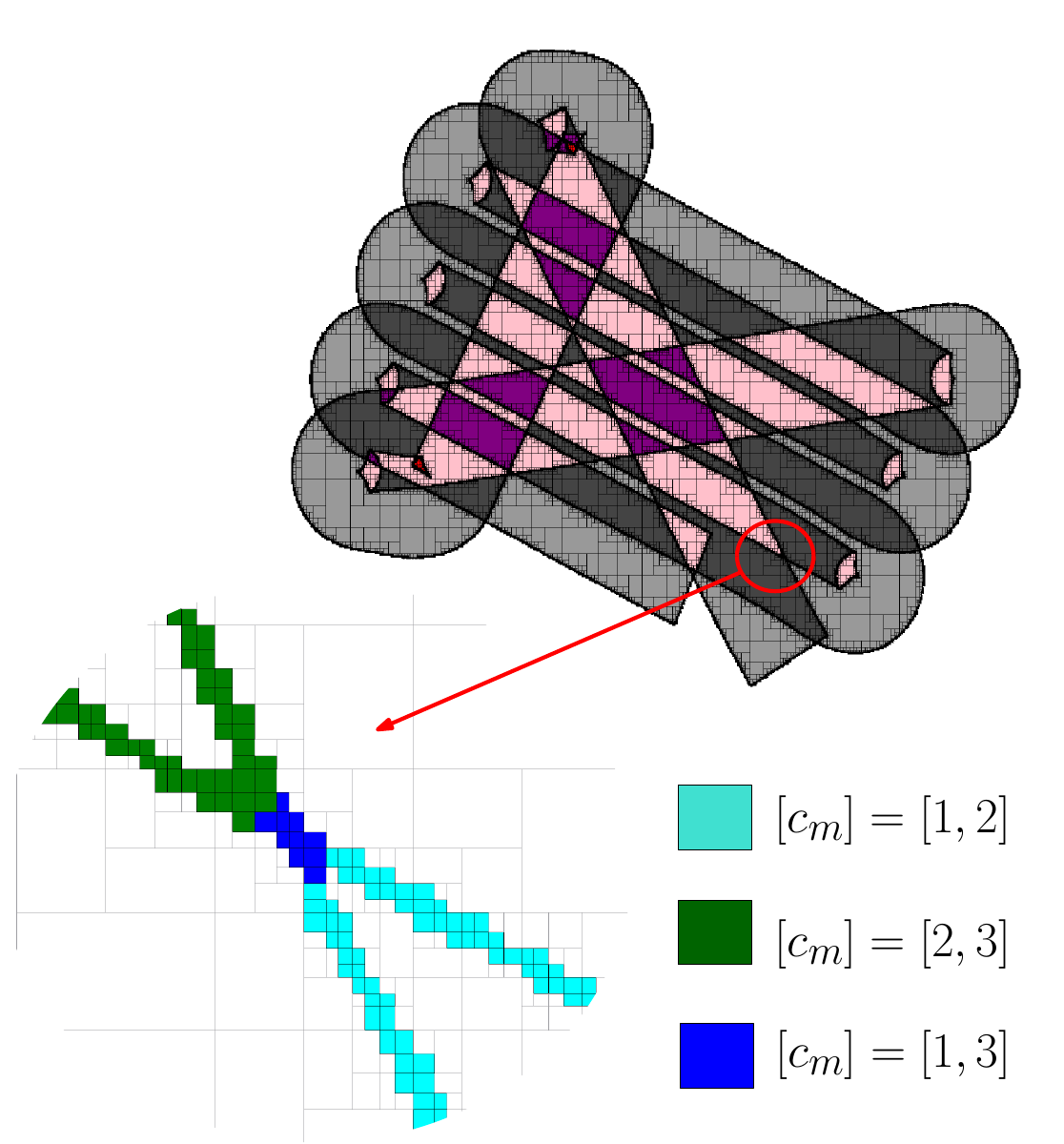}
\caption{Coverage measure for boxes that intersect the sensor's contour.}
\label{fig:13}
\end{figure}

The characterization of the explored area is done by calculating winding numbers $\eta(\tilde{\gamma}^+,\bm{p})$ and $\eta(\tilde{\gamma}^-,\bm{p})$ for all $\bm{p}$  inside the area considered of interest. The algorithm proposed in Section \ref{sec:compute} is used for this purpose. In Figure \ref{fig:12} we can see the resultant paving. Uncertain boxes, surrounding contours $\tilde{\gamma}^+$ and $\tilde{\gamma}^-$ are represented in black. The uncertain winding number value for each of these boxes can also be defined with the proposed algorithm, in Figure \ref{fig:13}, we give an overview of the classification of these boxes for a part of the mission. 

\begin{figure}[h]
  \centering
  \includegraphics[ trim={0.9cm 0.8cm 0 0}, clip, scale = 0.35]{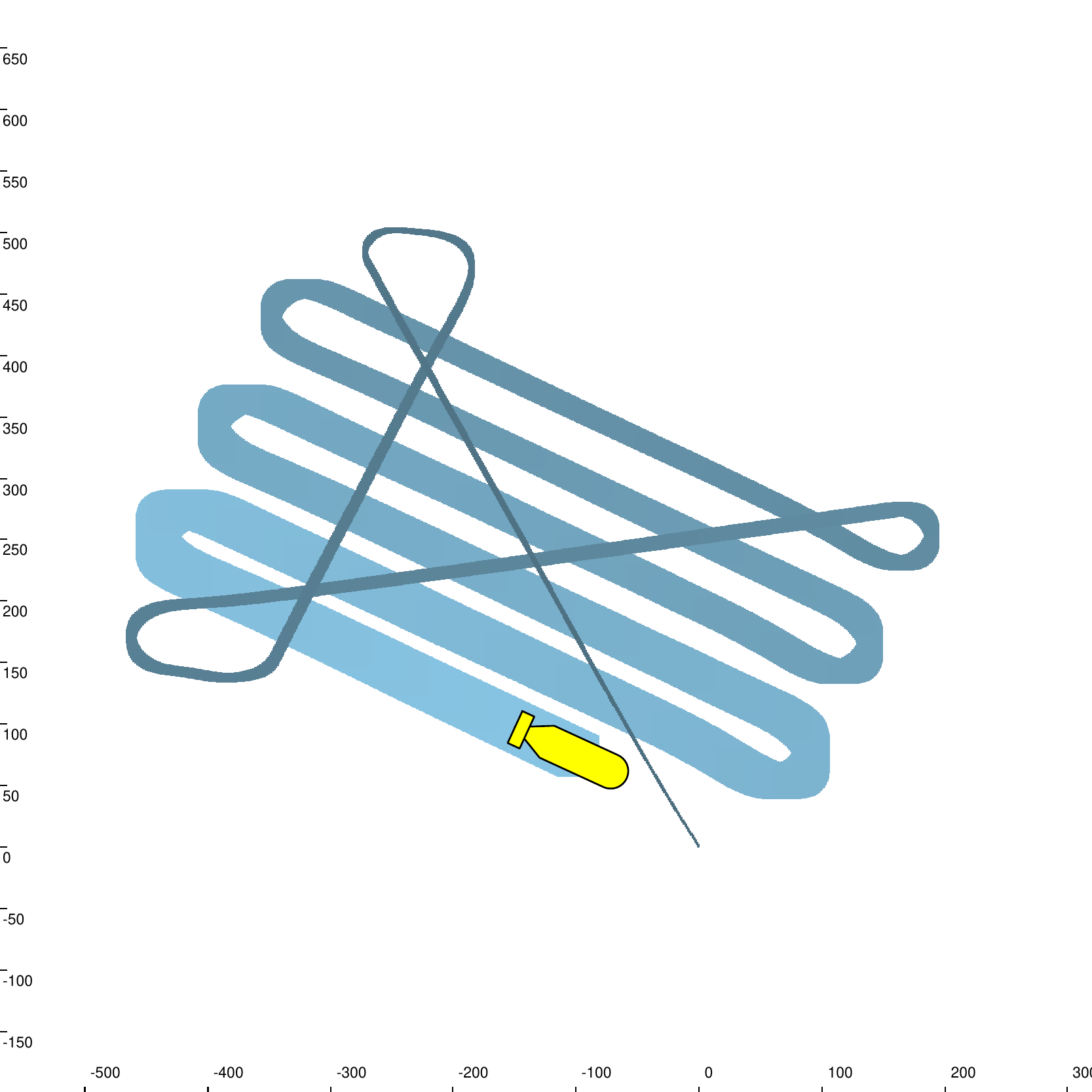}
    \caption{The inclusion function $[\bm{x}]$.}
  \label{fig:14sub1}
\end{figure}
\begin{figure}[h]
  \centering
  \includegraphics[ trim={0.9cm 0.8cm 0 0}, clip, scale = 0.35]{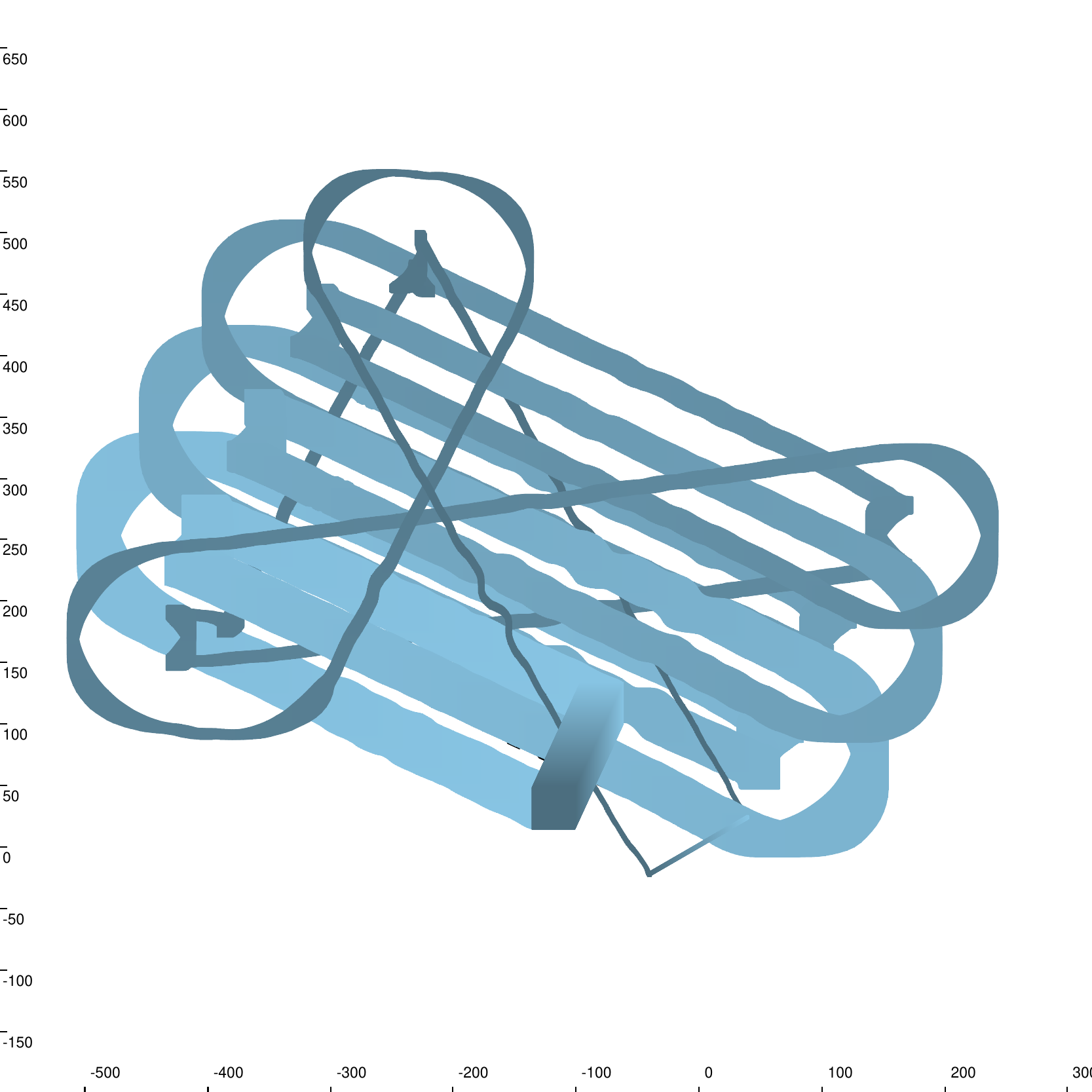}
  \caption{$[\gamma]$.}
  \label{fig:14sub2}

\end{figure}

Then, if we take into consideration the incertitude around sensors measurements, propagated through integration during pose estimation, we obtain $[\bm{x}]$, Figure \ref{fig:14sub1}. We represent the uncertain pose by a guaranteed envelope of the ground truth $\bm{x}^*$ using a box-valued function named tube on the interval analysis literature. The sonar's contour $[\gamma]$ will also be uncertain and represented by a tube, as displayed in Figure \ref{fig:14sub2}.

\begin{figure}[h]
\centering
  \includegraphics[scale = 0.5]{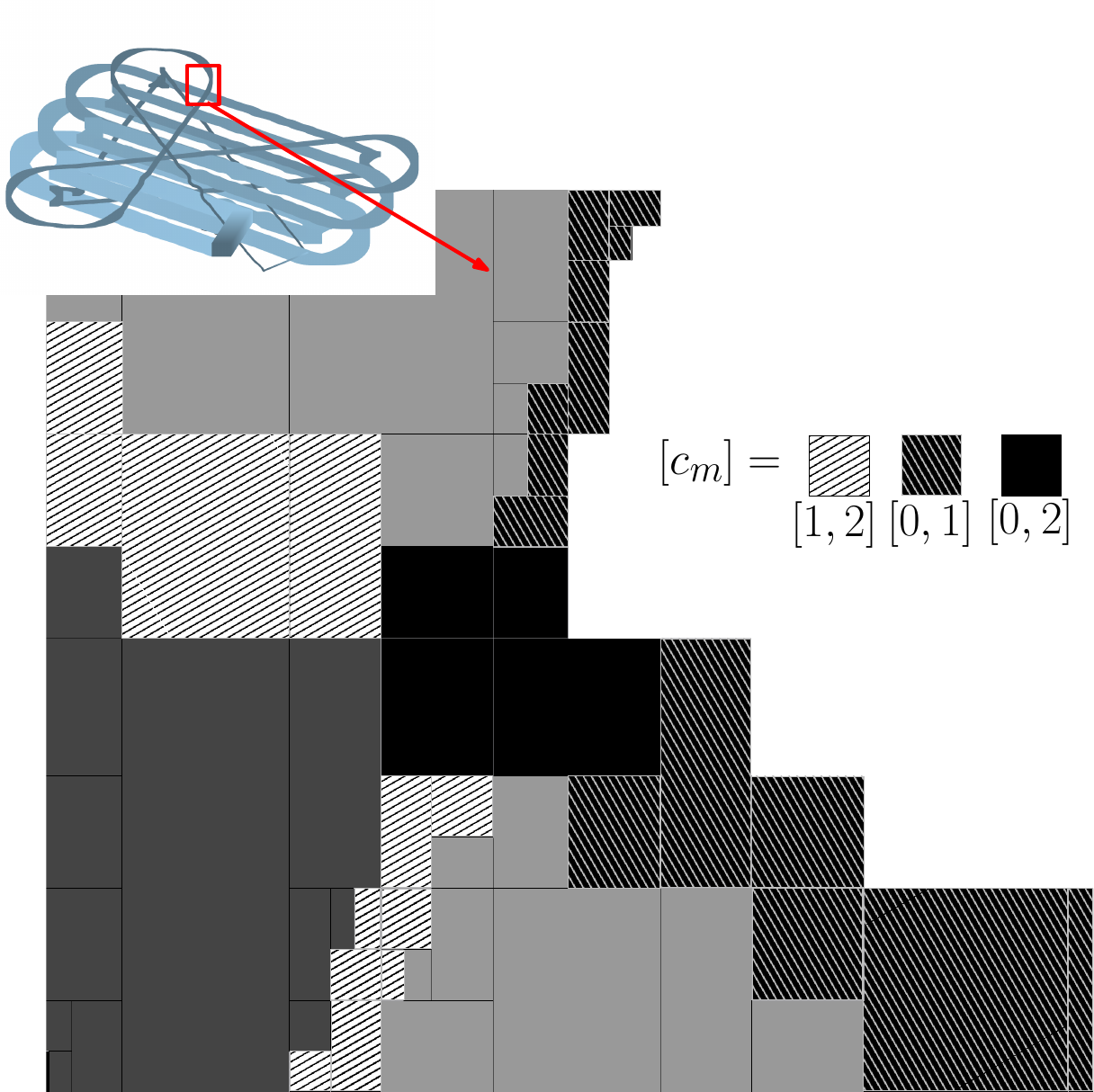}
\caption{Application of the uncertain Alexander Rule for one self-intersection that respects the conditions established by the presented method. Boxes in light gray are guaranteed to have been explored once and in dark gray twice. Other boxes have an uncertain coverage measure.}
\label{fig:15}
\end{figure}

In the considered scenario, some self-intersections of $[\gamma]$  do not respect the conditions established by our algorithm, notably, the non colinearity condition that ensures that the environment is divided into four regions around the self-intersection so the Alexander rules can be applied for numbering. As a result, the problem at hand cannot be directly solved using the proposed method. We apply, however, our algorithm around one uncertain self-intersection in $[\gamma]$ that respects our limitation in order to exemplify the extension of the Alexander algorithm to uncertain curves, as it was presented in Figure \ref{fig:alex4}. The result is illustrated in Figure \ref{fig:15}. One can note that the method presented in this paper can still be used to characterize the whole environment in this situation. For that, the mission must be divided into multiple missions, along the time of exploration, that respect individually the required constraints. 

\section{Conclusion}

In conclusion, this article has extended the link between the topological degree and the line-sweep exploration problem, allowing for a characterization of the area explored by a mobile robot in a two-dimensional plane. An interval analysis-based algorithm for computing the winding number for all the points inside a set has also been proposed, and its efficiency and scalability make it suitable for deployment on resource-constrained robotic platforms. A real-world experiment has shown that the proposed algorithm consistently produces reliable characterizations of the explored area, but it has also shown the limitations of the method that should be addressed by future work. Other future research directions may involve extending the algorithm to three-dimensional environments and exploration sensors with a two-dimensional visible area. Furthermore, the algorithm's applicability in collaborative multi-robot systems and its integration with simultaneous localization and mapping (SLAM) techniques could be explored. For the latter, we could imagine a scenario where the coverage measure is used to reduce the exteroceptive data that has to be compared to find possible feature matching, therefore, reducing the complexity of SLAM algorithms. 
Finally, we will examine the link between uncertain topological degrees and methods based on persistent homology, as in e.g. \cite{ghristcoveragepersistence}.

\section*{Acknowledgments}
We acknowledge the support of the "Engineering of Complex Industrial Systems" Chair Ecole Polytechnique-ENSTA Paris-T\'el\'ecom Paris, partially funded by DGA/AID, Naval Group, Thal\`es and Dassault Aviation.


\bibliographystyle{IEEEtran}
\bibliography{bib_1dim}

\vfill

\end{document}